\documentclass{article}

\usepackage{graphicx} 

\usepackage[utf8]{inputenc}
\usepackage[english]{babel}
\usepackage{amsmath,amsfonts,amssymb}
\usepackage{amsthm}
\newtheorem{theorem}{Theorem}
\newtheorem{lemma}{Lemma}
\newtheorem{corollary}{Corollary}

\RequirePackage[colorlinks=true, allcolors=blue]{hyperref}

\usepackage{sectsty}
\allsectionsfont{\raggedright}
\usepackage{jabbrv}

\title{Quantum-inspired probability metrics define a complete, universal space for statistical learning}

\author{Logan S. McCarty\\Department of Physics\\Harvard University, Cambridge MA USA\\ \href{mailto:mccarty@fas.harvard.edu}{\texttt{mccarty@fas.harvard.edu}}}
\begin{document}

\maketitle

\begin{abstract}
Comparing probability distributions is a core challenge across the natural, social, and computational sciences. Existing methods, such as Maximum Mean Discrepancy (MMD), struggle in high-dimensional and non-compact domains. Here we introduce quantum probability metrics (QPMs), derived by embedding probability measures in the space of quantum states: positive, unit-trace operators on a Hilbert space. This construction extends kernel-based methods and overcomes the incompleteness of MMD on non-compact spaces. Viewed as an integral probability metric (IPM), QPMs have dual functions that uniformly approximate all bounded, uniformly continuous functions on $\mathbb{R}^n$, offering enhanced sensitivity to subtle distributional differences in high dimensions. For empirical distributions, QPMs are readily calculated using eigenvalue methods, with analytic gradients suited for learning and optimization. Although computationally more intensive for large sample sizes ($O(n^3)$ vs. $O(n^2)$), QPMs can significantly improve performance as a drop-in replacement for MMD, as demonstrated in a classic generative modeling task. By combining the rich mathematical framework of quantum mechanics with classical probability theory, this approach lays the foundation for powerful tools to analyze and manipulate probability measures.
\end{abstract}

Given two probability measures on a set $X$, how can we measure their similarity? On this space $\mathcal{P}(X)$, one often chooses the \textit{weak topology}, which is the coarsest topology ensuring convergence of expected values for bounded continuous functions $f: X \to \mathbb{R}$ \cite{Merkle2000}. If there is a metric $d(x_1, x_2)$ defining the distance between points in $X$, we can construct a metric on $\mathcal{P}_\mathrm{w}(X)$ that \textit{metrizes} the weak topology (indicated by the subscript $\mathrm{w}$; we assume the Borel $\sigma$-algebra throughout). Many such metrics are \textit{integral probability metrics}, or IPMs. Given a family of functions $\mathcal{F}$, the IPM between two probability measures $\mu$ and $\nu$ is the supremum of the difference in expected values for all $f \in \mathcal{F}$: $ d(\mu, \nu) = \sup_{f\in\mathcal{F}} \big| \mathbb{E}_{x \sim \mu}[f(x)] - \mathbb{E}_{x \sim \nu}[f(x)] \big| $.

The functions in $\mathcal{F}$ are \textit{witness functions} for the IPM: they detect differences between probability measures. For example, when $\mathcal{F}$ is the set of functions with Lipschitz constant $L \leq 1$, the resulting metric is the 1-Wasserstein metric, or earth mover's metric. As this definition is not very useful for practical calculations, IPMs are often constructed by embedding $\mathcal{P}_\mathrm{w}(X)$ into a normed linear space $E$, such as a Banach or Hilbert space \cite{Muandet2017, Lin2022}. Given a continuous, injective function $\phi : X \to E$, define the affine barycenter map $T: \mathcal{P}_\mathrm{w}(X) \to E $ as the expected value of $\phi$:
$$ T(\mu) = \mathbb{E}_{x \sim \mu}[\phi(x)] \equiv \int_X \phi(x)\, d\mu(x) $$
For appropriate $\phi$, this integral is well-defined (as a Bochner or Pettis integral) and provides a topological embedding that respects the convex (affine) structure of $\mathcal{P}_\mathrm{w}(X)$. The norm on $E$ induces a metric, $ d(\mu, \nu) = \left\|T(\mu) - T(\nu)\right\| $, equivalently defined over the dual space $E^*$ of continuous linear functionals $\lambda$ that act on $E$:
$$ d(\mu, \nu) = \sup_{\substack{\lambda\in E^* \\ \|\lambda\| \leq 1}} \big| \lambda(T(\mu)) - \lambda(T(\nu)) \big|$$
The barycenter map $T$ plus linearity of the expected value gives: 
$$ d(\mu, \nu) = \sup_{\substack{\lambda\in E^* \\ \|\lambda\| \leq 1}} \big|\, \mathbb{E}_{x \sim \mu}[\lambda(\phi(x)] - \mathbb{E}_{x \sim \nu}[\lambda(\phi(x)] \,\big|$$
This is an IPM: $\lambda \in E^*$ defines a \textit{dual function} $\lambda(\phi(x))$ on $X$; the witness functions have $\|\lambda\| \leq 1$. Typically, calculating the norm of embedded probability measures is simpler than finding the supremum over an unwieldy class of functions. For instance, if we embed $\mathcal{P}_\mathrm{w}(X)$ into a Hilbert space $\mathcal{H}$ with inner product $\langle \cdot,\cdot \rangle$, then $\phi$ is the feature map for a reproducing kernel Hilbert space (RKHS) with kernel $k(x,y) = \langle\phi(x),\phi(y)\rangle$ \cite{Muandet2017}. With an appropriate kernel this embedding defines a metric on $\mathcal{P}_\mathrm{w}(X)$ known as the Maximum Mean Discrepancy (MMD). The dual functions are precisely the functions in the RKHS, with IPM witness functions contained in the unit ball.

\subsection*{Limitations of MMD}
The limited space of RKHS functions imposes limits on MMD. For the widely used Gaussian, Laplacian, or inverse-multiquadric kernels, all RKHS functions vanish at infinity; even constants are excluded. In addition, when the space $X$ is not compact, the MMD metric is \textit{incomplete}. Consider generating a sequence of probability measures $\mu_n$ by iterative sampling or modeling. Without an \textit{a priori} limit, we can only check if the sequence is Cauchy: do the terms get closer to one another as $n$ grows? In a complete metric space, every Cauchy sequence converges to a valid probability measure. But in an incomplete space, a seemingly convergent sequence may diverge to something invalid. This problem arises with embeddings into \textit{reflexive} Banach spaces, including all Hilbert spaces and most reproducing kernel Banach spaces \cite{Lin2022}. We prove in SI Appendix, \ref{reflexiveincomplete-proof}:
\begin{theorem}\label{reflexiveincomplete}
    For noncompact $X$, an affine embedding of $\mathcal{P}_\mathrm{w}(X)$ into a reflexive Banach space $E$ is never closed (in the norm topology). Thus, the induced metric is incomplete.
\end{theorem}

\subsection*{Quantum coherent states}
A complete metric needs a non-reflexive Banach space, such as the quantum state space. Elementary quantum states $|\psi\rangle$ are vectors in a Hilbert space $\mathcal{H}$, but a proper treatment uses the full state space $\mathcal{S}$ of positive, unit-trace operators on $\mathcal{H}$. The extreme points of this convex space are rank-1 projection operators $\hat{\rho} = |\psi\rangle\langle\psi|$ called \textit{pure states}. Their convex combinations form \textit{mixed states} such as $\hat{\rho} = \sum_i p_i |\psi_i\rangle\langle\psi_i|$ for a set of states $|\psi_i\rangle$ with probabilities $p_i$. This space $\mathcal{S}$ is contained in the \emph{non-reflexive} Banach space of trace-class operators $\mathcal{B}_1(\mathcal{H})$, with norm $\|\hat{A}\|_1 = \mathrm{tr}(|\hat{A}|)$, whose dual is the space of bounded operators $\mathcal{B}(\mathcal{H})$ with the operator norm. (We use Dirac notation and ``hats'' for operators; see SI Appendix, \ref{notation}, for details.)

A special class of quantum states are the \textit{coherent states} of the quantum harmonic oscillator \cite{Hall2000}. (In the math literature, this construction is known as the Fock space over $\mathbb{C}$.) Given a pair of complex numbers $w, z$, the corresponding coherent states satisfy $\langle w|z\rangle = \exp[-\frac{1}{2}(|w|^2 + |z|^2 - 2w^*z)] $. These states form a closed embedding of $\mathbb{C}$ in $\mathcal{H}$ with the map $\phi(z) = |z\rangle$, and define an RKHS with kernel $k(w,z) = \langle w|z\rangle$.

Although the coherent state kernel is not typically used for RKHS embeddings, its norm-square is precisely the Gaussian kernel on $\mathbb{C}$: $|k(w,z)|^2 = \exp(-|w-z|^2)$. This kernel arises naturally if we embed $z \in \mathbb{C}$ as a pure quantum state in $\mathcal{S}$, with embedding $\phi(z) = |z\rangle\langle z| \equiv \hat{\rho}_z$. The Hilbert-Schmidt inner product between two embedded states is $\mathrm{tr}(\hat{\rho}_\mathrm{w}\hat{\rho}_z) = \exp(-|w-z|^2)$. Extending to $\mathbb{C}^n$ gives (SI Appendix, \ref{isomorphism-proof}):
\begin{theorem}\label{isomorphism}
    Coherent states of an $n$-mode quantum harmonic oscillator, with the Hilbert-Schmidt inner product, are isometric to kernel functions in the Gaussian RKHS over $\mathbb{C}^n$.
\end{theorem}

\subsection*{Quantum probability metrics} 
For pure quantum states, the ``natural'' topology is induced equivalently by the trace norm and the Hilbert-Schmidt norm. For mixed quantum states, however, we will find the trace distance to be advantageous. It is conventionally normalized as $d(\hat{\rho}, \hat{\sigma}) = \frac{1}{2}\mathrm{tr}(|\hat{\rho}-\hat{\sigma}|)$ to take values in $[0,1]$. Given eigenvalues $\lambda_i$ of the compact self-adjoint operator $(\hat{\rho}-\hat{\sigma})$, the trace distance is $\frac{1}{2} \sum_i |\lambda_i|$.

We first build a Quantum Probability Metric (QPM) for measures on $\mathbb{C}$. Embed $z \in \mathbb{C}$ using the coherent-state map $\phi(z)=|z\rangle\langle z|$. The barycenter map $T: \mathcal{P}_\mathrm{w}(\mathbb{C}) \to \mathcal{S}$ gives:
$$ T(\mu) = \int_\mathbb{C} \phi(z)\, d\mu(z)=\int_\mathbb{C} \hat{\rho}_z\, d\mu(z)=\int_\mathbb{C} |z\rangle\langle z|\, d\mu(z) \equiv \hat{\mu} $$
where the ``hat'' signals the operator embedding of the measure. In physics this is known as the $P$-representation for the measure $d\mu(z) = P(z)d^2z$ on $\mathbb{C}$; this is not a contour integral, but rather over the 2D Lebesgue measure $d\Re z\; d\Im z$. From Theorem \ref{isomorphism}, this is precisely the Gaussian kernel embedding. The MMD is the Hilbert-Schmidt distance $d(\mu, \nu) = \mathrm{tr}(|\hat{\mu}-\hat{\nu}|^2)^{1/2}$, which metrizes $\mathcal{P}_\mathrm{w}(\mathbb{C})$, albeit with an incomplete metric. The Gaussian kernel is \textit{characteristic}, which means that the barycenter map is injective \cite{Muandet2017}. This construction, which we call the \textit{Fock embedding} of $\mathcal{P}_\mathrm{w}(\mathbb{C})$, extends to probability measures on $\mathbb{C}^n$, and even to measures on a separable infinite-dimensional Hilbert space \cite{ziegel2022}.

The QPM between two measures is the trace distance $d(\mu, \nu) = \frac{1}{2}\mathrm{tr}(|\hat{\mu}-\hat{\nu}|)$, which \textit{completely metrizes} $\mathcal{P}_\mathrm{w}(\mathbb{C})$ (i.e.\@ the metric is complete). Completeness relies on some remarkable properties of the barycenter map on a closed set of pure states $\mathcal{C} \subset\mathcal{S}$: it is a continuous, closed surjection from $\mathcal{P}_\mathrm{w}(\mathcal{C})$ onto the norm-closed convex hull of $\mathcal{C}$ \cite{Shirokov2007}; if it is injective, then it is a homeomorphism onto its image. Thus, given a characteristic kernel, the barycenter map yields a closed topological embedding of $\mathcal{P}_\mathrm{w}(\mathcal{C})$ into $\mathcal{S}$, and the trace distance completely metrizes the topology on $\mathcal{P}_\mathrm{w}(\mathcal{C})$.

We can now construct a general QPM. Choose $\phi : X \to \mathcal{S}$ to map points $x \in X$ to pure states $|x\rangle\langle x| \equiv \hat{\rho}_x$ on a Hilbert space. We require that $\phi$ is a closed embedding, so its image $\mathcal{C}$ is closed and homeomorphic to $X$, and also that its kernel $k(x,y) = \mathrm{tr}(\hat{\rho}_x\hat{\rho}_y)$ is characteristic. Embedding $\mathcal{P}_\mathrm{w}(X)$ into $\mathcal{S}$ using $\phi$ and the (injective) barycenter map, the QPM is the trace distance between the embedded measures. Adding some technical details in SI Appendix, \ref{completeness-proof}, we have:
\begin{theorem}\label{completeness}
    Every QPM completely metrizes $\mathcal{P}_\mathrm{w}(X)$.
\end{theorem}

All QPMs embed into the same state space $\mathcal{S}$, which is a universal embedding space: using \cite{Uspenskij2004, ziegel2022} in SI Appendix, \ref{universality-proof}, if $X$ is any separable, completely metrizable space---a Polish space---there is a map $\phi$ that meets the requirements of Theorem \ref{completeness}, so there is a QPM embedding into $\mathcal{S}$:
\begin{theorem}\label{universality}
    Every Polish space has a QPM.
\end{theorem}

\subsection*{Dual functions for a QPM}
Since every QPM is also an IPM, we can identify its dual functions using the duality between real Banach spaces of self-adjoint operators: $\mathcal{B}_{1,\mathrm{sa}}(\mathcal{H})$ (trace class) and  $\mathcal{B}_\mathrm{sa}(\mathcal{H})$ (bounded operators). In quantum mechanics, an operator $\hat{F}\in\mathcal{B}_\mathrm{sa}(\mathcal{H})$ is called an \textit{observable}; it is a linear functional on $\hat{\rho}\in\mathcal{B}_{1,\mathrm{sa}}(\mathcal{H})$ via the trace map, $\mathrm{tr}(\hat{\rho}\hat{F})$. Given a QPM that maps $x \in X$ to a pure state $\hat{\rho}_x = |x\rangle\langle x|$, every observable defines a real function $F(x) = \mathrm{tr}(\hat{\rho}_x \hat{F}) = \langle x|\hat{F}|x\rangle$. These are dual functions for the QPM; since $\hat{F}$ is bounded they are bounded and continuous. Embedding $\mu$ as a state $\hat{\mu} \in \mathcal{S}$, the (quantum) expectation value $\mathrm{tr}(\hat{\mu}\hat{F})$ equals the (classical) expected value $\mathbb{E}_{x\sim \mu}[F(x)]$. The IPM witness functions for a QPM have operator norm $\|\hat{F}\| \leq 1$. Even on compact domains, QPMs have a larger class of witness functions than MMD, which requires Hilbert-Schmidt norm $\|\hat{F}\|_2 \leq 1$.

On $\mathbb{R}^n$, the Fock embedding QPM has a rich class of dual functions that are not required to vanish at infinity, unlike those of MMD. In SI Appendix, \ref{uniform-approx-proof}, using the Schur test \cite{Grafakos2014}, we show that all bounded band-limited functions on $\mathbb{R}^n$ are included as dual functions; combined with \cite{Dryanov2003}, we prove:
\begin{theorem}\label{uniform-approx}
    Dual functions for the Fock (coherent state) QPM are dense in $BUC(\mathbb{R}^n)$: they can \textbf{uniformly} approximate any bounded, uniformly continuous function on $\mathbb{R}^{n}$.
\end{theorem}

\subsection*{Replacing MMD with a QPM} 
In most cases, MMD is calculated for atomic measures (convex combinations of Dirac measures). Although the underlying RKHS is usually infinite-dimensional, the MMD between two such measures can be calculated using kernel evaluations $k(x,y)$ \cite{Muandet2017}. If their joint support contains $n$ points, MMD requires the Gram matrix for every pair of embedded vectors, with complexity $O(n^2)$. 

The analogous QPM calculation needs eigenvalues of the rank-$n$ operator $(\hat{\mu}-\hat{\nu})$, which has the form $\sum_i c_i|x_i\rangle \langle x_i|$. Its eigenvalues are equal to those of the Hermitian matrix $M = H^\dagger C H$, where we factor the Gram matrix as $G=HH^\dagger$ and define $C=\mathrm{diag}(c_i)$. Although eigenvalue computation is $O(n^3)$, this is a simple ``drop-in'' replacement for MMD, perhaps with a modified kernel (SI Appendix, \ref{mmd2qpm}). Anyone who uses MMD should try this swap; the improved metric performance is likely worth the computational cost. Eigenvalue calculations are numerically stable with analytic gradients for optimization. By comparison, the widely-used $p$-Wasserstein metrics are also $O(n^3)$ with the Hungarian algorithm, but they are difficult to differentiate, need finite $p$\textsuperscript{th} moments, and have the opposite problem of MMD: sequences that diverge in the Wasserstein metric might converge in $\mathcal{P}_\mathrm{w}(X)$.

\begin{figure}[t!]
\centering
\includegraphics[width=\linewidth]{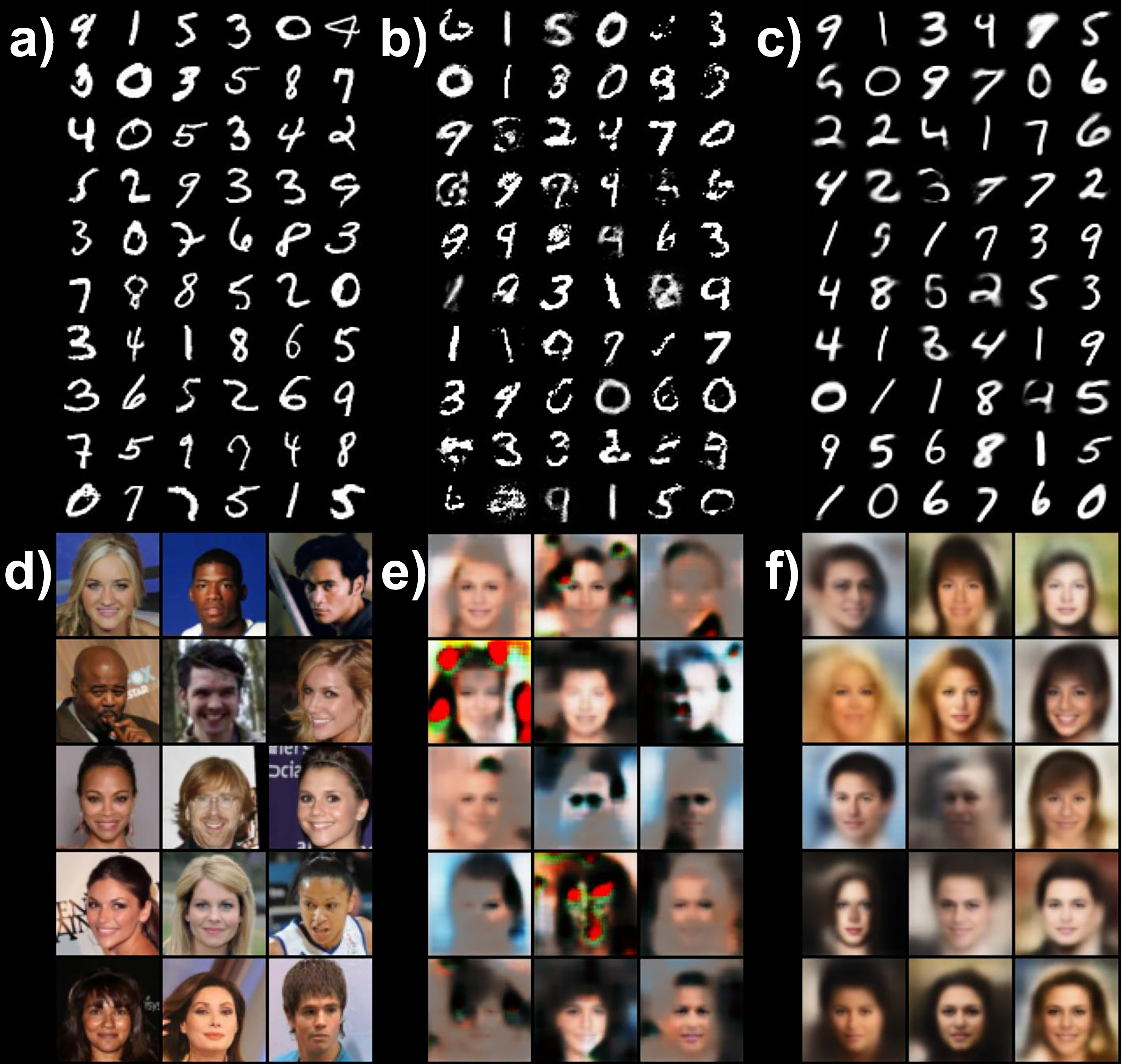}
\caption{(a) MNIST images vs.\@ those generated using (b) MMD or (c) QPM. (d) CelebA images vs.\@ those generated using (e) MMD or (f) QPM. Notably, an MMD two-sample test (same kernel as used for training) cannot distinguish (d) vs. (e). With batch sizes of 1000, MMD gives $p \approx 0.23$, while QPM gives $p < 10^{-3}$.}
\label{fig:generated}
\end{figure}

As a proof of concept, we swap MMD with QPM in a classic model, the Generative Moment Matching Network \cite{Li2015GMMN}; see SI Appendix, \ref{GMMN-application}, for details. These networks are trained to minimize the MMD (in the image space) between batches of generated images and batches drawn from the dataset. This is a great idea that gives poor results, as shown in Figure \ref{fig:generated} where we reproduce the original work on the MNIST dataset \cite{lecun1998gradient} (784 dimensions) and extend it to the CelebA-64 dataset \cite{liu2015faceattributes} (12,288 dimensions). The CelebA result illustrates why this approach fails: MMD believes that it has met its objective, as shown by an MMD two-sample test \cite{Muandet2017} that finds no statistically significant difference between the generated images and the data ($p \approx0.23$). The MMD witness functions, which vanish at infinity, simply fail in this high-dimensional case; a different architecture or training method would be unlikely to help. Replacing MMD with a QPM, keeping everything else constant, yields much better images. And QPM knows that it has not met its objective: its two-sample test gives $p < 10^{-3}$. With its richer class of witness functions on this compact, high-dimensional space, QPM easily distinguishes the generated images from the data. This experiment highlights QPM as a target for generative modeling and a discerning metric for probability distributions.

\subsection*{Future work} 
Using QPMs, we can study and manipulate probability measures with powerful tools from quantum mechanics such as unitary transformations, group representations, tensor products, entropy, information theory, and more. In addition, this construction offers an interesting independent motivation for the mathematical structure of quantum mechanics.

\subsection*{Acknowledgments}
The AI model ``GPT-5 Thinking'' from OpenAI \cite{openai2025} assisted with checking proofs in the SI Appendix.

\bibliographystyle{pnas-new}
\bibliography{qpms}

\clearpage

\section*{SI Appendix, Materials and Methods}

\subsection{Proofs}\label{proofs}
\subsubsection{Theorem \ref{reflexiveincomplete}}\label{reflexiveincomplete-proof}
\emph{For noncompact $X$, an affine embedding of $\mathcal{P}_\mathrm{w}(X)$ into a reflexive Banach space $E$ is never closed (in the norm topology). Thus, the induced metric is incomplete.}

Note that if $\mathcal{P}_\mathrm{w}(X)$ can be embedded in a Banach space, then $\mathcal{P}_\mathrm{w}(X)$ is metrizable, so $X$ is metrizable via the canonical, continuous embedding that maps $x \in X$ to its point mass $\delta_x$. We begin with two preparatory lemmas.

\begin{lemma}\label{notclosed}
    $\mathcal{P}_\mathrm{w}(X)$ is not closed for noncompact, metrizable $X$.
\end{lemma}

\begin{proof}
    We review results from \cite{Merkle2000}, which proves this lemma slightly differently. For noncompact $X$, the dual of $C_b(X)$ (the Banach space of bounded continuous real functions) is the space of signed, finite, regular, \textit{finitely} additive measures on $X$; we denote these generalized measures as $\mathcal{M}(X)$. The probability measures $\mathcal{P}(X)$ are a proper subset of $\mathcal{M}(X)$: they are positive, \textit{countably} additive, and normalized to one. The ``weak topology'' on $\mathcal{P}_\mathrm{w}(X)$ is the subspace topology inherited from the weak-* topology on $\mathcal{M}(X)$.

    Using sequential noncompactness, the metrizable space $X$ has a countable closed subset $S$ with no accumulation points; we index it as $S = \{x_n : n \in \mathbb{Z} \}$. Since $C_b(S) \cong \ell^\infty(\mathbb{Z})$, its dual can be identified with the space of generalized measures $\mathcal{M}(S) \cong \mathcal{M}(\mathbb{Z})$, or the \textit{ba space} over $\mathbb{Z}$. The probability measures $\mathcal{P}_\mathrm{w}(S)$ correspond to the unit simplex in $\ell^1(\mathbb{Z})$, with the weak-* topology inherited from $\mathcal{M}(S)$. Note that $\mathcal{M}(S)$ is a weak-*-closed subset of $\mathcal{M}(X)$: define $I_S=\{f\in C_b(X): f|_S\equiv 0\}$, then $\mathcal M(S)=\bigcap_{f\in I_S}\{\mu\in\mathcal M(X): \mu(f)=0\}$, an intersection of weak-*-closed sets.

    Define the left shift operator $L$ on $\mathcal{M}(S)$: for any index set $I \subseteq \mathbb{Z}$, $L\mu(\{x_k : k \in I\}) = \mu(\{x_k : k+1 \in I\})$. Build a shift-invariant measure as follows: For each $n \in \mathbb{N}$, define $S_n = \{x_k : -n \leq k \leq n \}$ and let $\mu_n$ be the uniform probability measure on $S_n$. By Banach-Alaoglu, some subnet $\mu_{n_\alpha}$ converges weak-* to a measure $\mu \in \mathcal{M}(S)$. A short calculation using the F{\o}lner property of $S_n$ shows that $\mu$ is weakly shift-invariant (i.e.\@ $L\mu(f) = \mu(f)$ for $f \in C_b(S)$). By applying the shift $m$ times we see that the singleton measure $\mu(\{x_k\}) = \mu(\{x_{k+m}\})$ for any $k$, which implies $\mu(\{x\})$ is constant for all $x \in S$.

    Yet a probability measure on a countable set cannot be shift-invariant. Countable additivity forces either $\mu(S) = 0$, if $\mu(\{x\}) = 0$, or $\mu(S) = \infty$, if $\mu(\{x\}) > 0$, contradicting $\mu(S) = 1$. So $\mu$ is not countably additive (just like a Banach limit) and $\mu \notin \mathcal{P}_\mathrm{w}(S)$. This shows $\mathcal{P}_\mathrm{w}(S)$ is not closed. Now $\mathcal{P}_\mathrm{w}(S)$ = $\mathcal{P}_\mathrm{w}(X) \cap \mathcal{M}(S)$. Since $\mathcal{M}(S)$ is weak-* closed, but $\mathcal{P}_\mathrm{w}(S)$ is not, $\mathcal{P}_\mathrm{w}(X)$ is not closed.
\end{proof}

\begin{lemma}\label{extension}
    Any continuous affine map $T: \mathcal{P}_\mathrm{w}(X) \to E$ (norm topology on the Banach space $E$) has a canonical linear extension to $T : \mathcal{M}(X) \to E^{**}$ that is continuous from the weak-* topology on $\mathcal{M}(X)$ to the weak-* topology on the double dual $E^{**}$.
\end{lemma}

\begin{proof}
    The central idea is that $T$ is determined by its action on the point measures $\delta_x$, whose convex combinations are dense in $\mathcal{P}_\mathrm{w}(X)$ (Theorem 4.3 in \cite{Merkle2000}). Their images $T(\delta_x)$ are bounded in $E$: suppose, for contradiction, that they are unbounded. Then there is a sequence $x_n$ with $\|T(\delta_{x_n})\| > n, \forall n \in \mathbb{N}$. Construct a sequence of probability measures $\mu_n = \tfrac{1}{n}\delta_{x_n} + (1-\tfrac{1}{n})\delta_{x_1} $, which converges in $\mathcal{P}_\mathrm{w}(X)$ to $\delta_{x_1}$. Yet the image of this sequence under the affine map $T$ does not converge to $T(\delta_{x_1})$:
    \begin{align*}
        \|T(\mu_n)-T(\delta_{x_1})\| &= \|\tfrac{1}{n}T(\delta_{x_n}) + (1-\tfrac{1}{n})T(\delta_{x_1})-T(\delta_{x_1})\| \\
        &= \tfrac{1}{n}\|T(\delta_{x_n})-T(\delta_{x_1})\|  \\
        &\geq  \tfrac{1}{n}\|T(\delta_{x_n})\|-\tfrac{1}{n}\|T(\delta_{x_1})\|   \\
        &> 1-\tfrac{1}{n}\|T(\delta_{x_1})\|\xrightarrow[n\to\infty]{}1
    \end{align*}
    which contradicts continuity of $T$. In addition, $T(\delta_x)$, viewed as a function of $x$, is continuous from $X$ to $E$, as it is the composition of the continuous embedding $x \to \delta_x$ with the continuous map $T$.
    
    Now extend $T(m)$ for a generalized measure $m \in \mathcal{M}(X)$. Given a linear functional $\lambda \in E^*$, define $T(m) \in E^{**}$ by its action on $\lambda$:
    $$ T(m)[\lambda] = \int_X \lambda(T(\delta_x))\, dm(x)$$
    Note that this ``integral'' of the real-valued function $\lambda(T(\delta_x))$ over the generalized measure $m$ arises from viewing $m \in \mathcal{M}(X)$ as an element of the continuous dual of $C_b(X)$; it is not a Lebesgue integral since $m$ may not be countably additive. This expression defines a linear functional of $\lambda \in E^*$ that is also linear in $m$. We show that it is continuous in $\lambda$ (making it an element of $E^{**}$), and that it is continuous in $m$ (weak-* to weak-* topology).

    Continuity in $\lambda$ follows from boundedness via the following calculation.  Note that $|m|$ is the absolute value $(m_+ - m_-)$ of the signed measure $m$, while $\|m\|$ is its total variation norm---the dual norm for $\mathcal{M}(X) \cong C_b(X)^*$:
    \begin{align*}
        \big|T(m)[\lambda]\big|&= \left| \int_X \lambda(T(\delta_x))\, dm(x) \right| 
          \leq  \int_X \big|\lambda(T(\delta_x))\big| \, d|m|(x) \\
          &\leq  \|\lambda\| \int_X \| T(\delta_x) \| \, d|m|(x) 
          \leq \|\lambda\| \|m\| \sup_{x\in X} \| T(\delta_x) \| 
    \end{align*}
    where the final bound follows from boundedness of $T(\delta_x) \in E$ and finite total variation of $m \in \mathcal{M}(X)$. Hence $T(m) \in E^{**}$.

    Continuity in $m$ follows from the observation that $\lambda(T(\delta_x)) \in C_b(X)$: it is the composition of continuous maps, and it is bounded since $T(\delta_x)$ is bounded. Thus, given a weak-* convergent net of generalized measures $m_\alpha \to m$, we find:
    \begin{align*}
        \lim_\alpha\, T(m_\alpha)[\lambda]&= \lim_\alpha \int_X \lambda(T(\delta_x))\, dm_\alpha(x) \\
        &= \int_X \lambda(T(\delta_x))\, dm(x) = T(m)[\lambda]
    \end{align*}
    by weak-* convergence in $\mathcal{M}(X)$ for $\lambda(T(\delta_x)) \in C_b(X)$.

    Using the affine property of $T$, linearity of $\lambda$, and linearity of the integral, this extended $T$ agrees with the original $T(\mu)$ for convex combinations of point measures, which are dense in $\mathcal{P}_\mathrm{w}(X)$. By weak-* continuity, therefore, it agrees on all of $\mathcal{P}_\mathrm{w}(X)$.

    Note that this construction can also be viewed as the \textit{Dunford integral} of the $E$-valued map $T(\delta_x)$ over the generalized measure $m$; such ``integrals'' generally lie in the double dual $E^{**}$.
\end{proof}
We now return to the main theorem.
\begin{proof} (Theorem \ref{reflexiveincomplete}, SI Appendix, \ref{reflexiveincomplete-proof})
Assume we have an affine closed embedding map $T$ (a homeomorphism from the weak topology on $\mathcal{P}_\mathrm{w}(X)$ to its norm-closed image in $E$); its inverse $T^{-1}$ is also affine. Since $\mathcal{P}_\mathrm{w}(X)$ is convex, its image is convex. Using Lemma \ref{notclosed}, choose a net $\mu_\alpha \in \mathcal{P}_\mathrm{w}(X)$ that is weak-* convergent to a shift-invariant generalized measure $\mu$. By Lemma \ref{extension}, the image of this net $T(\mu_\alpha)$ converges weak-* to $T(\mu) \in E^{**}$. Since $E$ is reflexive, weak-* convergence is equivalent to weak convergence, and $E^{**} \cong E$, so $T(\mu_\alpha)$ converges weakly to $T(\mu) \in E$. Define $v_\alpha \equiv T(\mu_\alpha)$ as its image in $E$, and $v \equiv T(\mu)$ as its weak limit.

Using Mazur's Lemma for nets, construct $w_{\alpha, n}$ on the directed set defined by $(\alpha',n') \geq (\alpha,n)$ when $\alpha' \geq \alpha$ and $n' \geq n$, such that (i) each $w_{\alpha, n}$ is a finite convex combination of tails of $v_\alpha$, i.e.\@ from $\{v_\beta : \beta \geq \alpha\}$, and (ii) $\|w_{\alpha, n}-v\|\leq 1/n$, so $w_{\alpha, n}$ converges strongly to $v = T(\mu)$ in $E$. By assumption, this net and its (strong) limit lie within the norm-closed convex image of $\mathcal{P}_\mathrm{w}(X)$.

The affine inverse $T^{-1}$ defines a net $\nu_{\alpha,n} \equiv T^{-1}(w_{\alpha,n})$; these are probability measures since they are convex combinations of tails of $\mu_\alpha$. Now $T^{-1}$ is continuous from the norm topology on $E$ to the weak topology on $\mathcal{P}_\mathrm{w}(X)$, so the net $\nu_{\alpha,n}$ converges in $\mathcal{P}_\mathrm{w}(X)$ to $\nu \equiv T^{-1}(v)$; it also converges in the ambient weak-* topology on $\mathcal{M}(X)$. Since each $\nu_{\alpha,n}$ is built from tails of $\mu_\alpha$, both nets must converge to the same limit in this topology. Yet the limit of $\mu_\alpha$ is the shift-invariant generalized measure $\mu$, which cannot be a probability measure by Lemma \ref{notclosed}. This contradicts the assumption that the image of $\mathcal{P}_\mathrm{w}(X)$ in $E$ is norm-closed.

Note that reflexivity of $E$ is essential. Otherwise, the image $T(\mu)$ of the shift-invariant generalized measure $\mu$ might not lie within $E$; the extension of $T$ only ensures it lies in $E^{**}$, and it is merely a weak-* limit, not a weak limit. Moreover, Mazur's Lemma can ``upgrade'' weakly-convergent nets to norm-convergent nets using convex combinations, but this does not work for weak-*-convergent nets, so the key contradiction does not hold for nonreflexive $E$.
\end{proof}

As a corollary, we note that an embedding of $\mathcal{P}_\mathrm{w}(X)$ that is not closed can lead to pathological features such as including the \textit{zero vector} in the closure of the embedded image. Moreover, since the image is convex, including zero means that images of subnormalized positive measures with $\mu(X) < 1$ are also included. This is the case for the usual Kernel Mean Embedding of $\mathcal{P}_\mathrm{w}(X)$ for many widely-used kernels, such as the Gaussian, Laplacian, or inverse-multiquadric kernels on $\mathbb{R}^n$ \cite{Muandet2017}, because these kernels vanish at infinity:

\begin{corollary}
    For noncompact, locally-compact $X$, the zero vector is contained in the closure of the image of $\mathcal{P}_\mathrm{w}(X)$ under the Kernel Mean Embedding into any RKHS of bounded functions that vanish at infinity.
\end{corollary}

\begin{proof}
    Recall that the Kernel Mean Embedding is the affine map $T : \mathcal{P}_\mathrm{w}(X) \to \mathcal{H}$ defined by $T(\delta_x) = K_x$ for the kernel function $K_x(\,\cdot\,) \equiv K(\,\cdot\, , x) \in \mathcal{H}$ \cite{Muandet2017}. By Lemma \ref{extension}, this map can be extended to all generalized measures, and by the argument in Theorem \ref{reflexiveincomplete} (SI Appendix, \ref{reflexiveincomplete-proof}), given any weak-* convergent net of probability measures $\mu_\alpha \to \mu$, the image of the limit $T(\mu)$ is contained in the norm closure of the image of $\mathcal{P}_\mathrm{w}(X)$.
    
    Using the construction in Lemma \ref{notclosed}, choose a weak-* convergent subnet $\mu_{n_\alpha}$ that converges to a shift-invariant generalized measure $\mu$ supported on the discrete set $S$. Choose any function $f \in \mathcal{H}$ and evaluate the inner product with $T(\mu_{n_\alpha})$ using the reproducing property of the kernel:
    \begin{align*}
        \lim_\alpha\, \big<T(\mu_{n_\alpha}),f\big> &= \lim_\alpha \int_X \big<T(\delta_x),f\big>\, d\mu_{n_\alpha}(x) \\
        &= \lim_\alpha\, \int_X \big<K_x,f\big>\, d\mu_{n_\alpha}(x) \\
        &= \lim_\alpha\, \int_X f(x)\, d\mu_{n_\alpha}(x)
    \end{align*}
    We show that this expression converges to zero for any $f \in \mathcal{H}$, which implies that the limit $T(\mu)$ must vanish. Choose $\epsilon>0$. Since $f$ vanishes at infinity, there is a compact set $C \subset X$ such that $|f(x)| < \epsilon/2$ for any $x \in X \backslash C$. Now recall that $\mu_{n_\alpha}$ is a subnet of the sequence $\mu_n$, and each $\mu_n$ is the uniform probability measure on the set $S_n = \{x_k : -n \leq k \leq n \}$. Let $N$ be the number of points in $S$ contained in $C$; it must be finite since $S$ has no accumulation points. The probability mass inside $C$ vanishes at the rate $\mu_n(C) \leq N /(2n+1)$, which gives the bound:
    $$ \left|\int_C f(x)\, d\mu_n(x)\right| \; \leq \; \frac{NM}{2n+1} $$
    for $M = \sup |f(x)|$, which is bounded for $f \in \mathcal{H}$. Choose $n$ so this integral is less than $\epsilon/2$; we already know that the integral of $|f|\, d\mu_n$ over $X \backslash C$ is less than $\epsilon/2$. So the sequence $\langle T(\mu_n),f\rangle$ converges to zero, and thus any subnet $\langle T(\mu_{n_\alpha}),f\rangle$ also converges to zero. Hence the weak limit (zero vector) is contained in the weak closure, and hence also in the norm closure, of the (convex) image of $\mathcal{P}_\mathrm{w}(X)$ (recall Hahn-Banach implies that weak- and norm-closures coincide). This reflects a failure of tightness---in the sense of Prokhorov's theorem---since for noncompact X probability measures can ``escape to infinity,'' which causes the expected value to vanish for any function that vanishes at infinity.
\end{proof}

\subsubsection{Theorem \ref{isomorphism}}\label{isomorphism-proof}
\emph{Coherent states of an $n$-mode quantum harmonic oscillator, with the Hilbert-Schmidt inner product, are isometric to kernel functions in the Gaussian RKHS over $\mathbb{C}^n$.}

We first review the well-known Segal-Bargmann-Fock correspondence. Unlike the usual Bargmann representation, which uses holomorphic functions, we follow Segal and choose anti-holomorphic functions \cite{Hall2000}:

\begin{lemma}\label{bargmann}
    The Hilbert space of an $n$-mode quantum harmonic oscillator is isometrically isomorphic to the RKHS of complex square-integrable functions over $\mathbb{C}^n$ of the form $f(\mathbf{z}) = \exp(-|\mathbf{z}|^2 / 2)\overline{F(\mathbf{z})}$, where $F$ is an entire function of $\mathbf{z}$.
\end{lemma}

\begin{proof}
    The Hamiltonian of the $n$-mode quantum harmonic oscillator has a complete set of orthonormal eigenstates known as Fock states indexed by $n$-dimensional vectors of nonnegative integers; we denote these states as $|\mathbf{n}\rangle \equiv |n_1, n_2, \ldots, n_n\rangle$. Each complex vector $\mathbf{z} \in \mathbb{C}^n$ defines a normalized coherent state \cite{Hall2000}:
    \begin{equation}\label{coherent-def}
        |\mathbf{z}\rangle = e^{-|\mathbf{z}|^2 / 2} \sum_{\mathbf{n}=\mathbf{0}}^\mathbf{\infty}(\mathbf{n}!)^{-1/2} \mathbf{z}^\mathbf{n} |\mathbf{n}\rangle 
    \end{equation}
    where $\mathbf{n}!=n_1!n_2!\ldots n_n!$ and $\mathbf{z}^\mathbf{n}=z_1^{n_1}z_2^{n_2}\ldots z_n^{n_n}$ are standard multi-index notation, and we sum over all $\mathbf n\in\mathbb N_0^n$. A calculation using the orthonormality of Fock states gives the inner product of two coherent states (and confirms $\langle\mathbf{z}|\mathbf{z}\rangle = 1$):
    \begin{equation}\label{inner-prod}
    \langle\mathbf{w}|\mathbf{z}\rangle = e^{- \frac{1}{2}(|\mathbf{w}|^2 + |\mathbf{z}|^2 - 2 \overline{\mathbf{w}}\mathbf{z}) }
    \end{equation}
    where $\overline{\mathbf{w}}\mathbf{z}$ is the inner product in $\mathbb{C}^n$. We denote the Hilbert space of the $n$-mode harmonic oscillator as $\mathcal{H}_F$; the coherent states are a total set in this space. This construction of the Fock space over $\mathbb{C}^n$ can be performed over any separable Hilbert space \cite{Hall2000}.

    Every vector $|\psi\rangle \in \mathcal{H}_F$ defines a unique function $\psi(\mathbf{z})$ on $\mathbb{C}^n$:
    $$\psi(\mathbf{z}) \equiv \langle \mathbf{z}|\psi\rangle = e^{-|\mathbf{z}|^2 / 2} \sum_\mathbf{n}(\mathbf{n}!)^{-1/2}  \langle\mathbf{n}|\psi\rangle\overline{\mathbf{z}}^\mathbf{n}$$
    where the power series converges for all $\mathbf{z}$ by boundedness of $|\langle\mathbf{n}|\psi\rangle| \leq \|\psi\|$ and rapid decrease of the prefactor $(\mathbf{n}!)^{-1/2}$. This function has the desired form: a Gaussian weight $\exp(-|\mathbf{z}|^2 / 2)$ times the complex conjugate of an entire function of $\mathbf{z}$. Uniqueness follows from uniqueness of the power series coefficients, which give the expansion of $|\psi\rangle$ in the basis of Fock states. Convergence of this power series ensures weak convergence in $\mathcal{H}_F$ of the power-series definition of the coherent state $|\mathbf{z}\rangle$, which in turn ensures strong convergence since these states are normalized to 1.
    
    The following identity (readily derived in polar coordinates) shows that the monomials $z^n$ are orthogonal when integrated over $\mathbb{C}$ with Gaussian weight:
    $$ \int e^{-|z|^2}z^n\overline{z}^m \,d^2z = \pi\,n!\,\delta_{n,m}$$
    This identity links the orthogonality of single-mode Fock states with orthogonality of the corresponding monomials. In $\mathbb C^n$ the integral factorizes, yielding $\pi^n\prod_k n_k!\,\delta_{\mathbf n,\mathbf m}$, so one can easily show that the inner product in $\mathcal{H}_F$ is given by an integral over the corresponding functions:
    \begin{equation}\label{bargmann-inner-prod}
        \langle\phi|\psi\rangle = \frac{1}{\pi^n}\int \overline{\phi(\mathbf{z})}\psi(\mathbf{z})\, d^{2n}\mathbf{z} = \frac{1}{\pi^n}\int \langle\phi|\mathbf{z}\rangle\langle\mathbf{z}|\psi\rangle\, d^{2n}\mathbf{z}
    \end{equation}
    where $d^{2n}\mathbf{z}$ is the $2n$-dimensional Lebesgue measure over the real and imaginary components of $\mathbf{z}$. This inner product defines an RKHS of functions on $\mathbb{C}^n$. Conversely, any function of the form $\exp(-|\mathbf{z}|^2 / 2)\overline{F(\mathbf{z})}$ is uniquely determined by the power-series coefficients of the entire function $F(\mathbf{z})$. Using the coherent-state expansion, one can extract the unique Fock-basis coefficients of the corresponding vector in $\mathcal{H}_F$. Thus we have an isometric bijection between $\mathcal{H}_F$ and the desired coherent state RKHS.

    The kernel for this RKHS is simply $k(\mathbf{w},\mathbf{z}) = \langle\mathbf{w}|\mathbf{z}\rangle$. This kernel is strictly positive definite (s.p.d.), which means that for any distinct $\mathbf{z}_i$, the matrix $M_{ij} = k(\mathbf{z}_i,\mathbf{z}_j)$ is an s.p.d.\@ matrix \cite{Muandet2017}. This s.p.d.\@ Gram matrix $\langle\mathbf{z}_i|\mathbf{z}_j\rangle$ implies linear independence for any distinct, finite set of coherent states.
    
    Although this RKHS may look like $L^2(\mathbb{C}^n)$, it only contains functions of the form $\exp(-|\mathbf{z}|^2 / 2)\overline{F(\mathbf{z})}$, with $F$ entire. This is a closed subspace of $L^2(\mathbb{C}^n)$ \cite{Hall2000}: briefly, $L^2$ convergence of Gaussian-weighted entire functions implies uniform convergence on compact sets, which yields a Gaussian-weighted entire limit by Morera's theorem. The expansion of the inner product in coherent states suggests the well-known resolution of the identity operator:
    \begin{equation}\label{resolution-ident}
        \hat{I} \simeq \frac{1}{\pi^n}\int |\mathbf{z}\rangle\langle\mathbf{z}|\, d^{2n}\mathbf{z} 
    \end{equation}
    where the notation $\simeq$ means equality in the sense of the weak operator topology, i.e., it is an equality when ``sandwiched'' between any pair of states $\langle\phi|\cdot|\psi\rangle$.
\end{proof}

    We now return to the main theorem.

\begin{proof}  (Theorem \ref{isomorphism}, SI Appendix, \ref{isomorphism-proof})  
    Lemma \ref{bargmann} connects the space $\mathcal{H}_F$ with the coherent state RKHS of square-integrable functions over $\mathbb{C}^n$. We now connect this coherent state RKHS with the Gaussian RKHS over $\mathbb{C}^n$. 
    
    Every RKHS is uniquely defined by its kernel (Theorem 2.5 in \cite{Muandet2017}); the Gaussian kernel is $K(\mathbf{w},\mathbf{z}) = \exp(-|\mathbf{w}-\mathbf{z}|^2)$. A short calculation using the expression for the inner product of coherent states (eq.\@ \ref{inner-prod} of Lemma \ref{bargmann}) shows that the Gaussian kernel is given by the square of the coherent state kernel:
    $$K(\mathbf{w},\mathbf{z}) = |k(\mathbf{w},\mathbf{z})|^2 = \overline{k(\mathbf{w},\mathbf{z})}k(\mathbf{w},\mathbf{z}) = \langle\mathbf{z}|\mathbf{w}\rangle\langle\mathbf{w}|\mathbf{z}\rangle $$
   As the coherent state kernel is strictly positive definite, so is the Gaussian kernel: its matrix is the Hadamard product of the individual kernel matrices, so strict positive definiteness follows by the Schur product theorem.

    Define the pure coherent states $\hat{\rho}_\mathbf{z} =|\mathbf{z}\rangle\langle\mathbf{z}| $. These rank-one operators belong to the Hilbert-Schmidt space $\mathcal{B}_2(\mathcal{H}_F)$ of operators with inner product $(\hat{A},\hat{B}) \equiv \mathrm{tr}(\hat{A}^\dagger\hat{B})$. The inner product between two coherent states is:
    $$ (\hat{\rho}_\mathbf{w},\hat{\rho}_\mathbf{z}) = \mathrm{tr}(\hat{\rho}_\mathbf{w}\hat{\rho}_\mathbf{z}) = \mathrm{tr}(|\mathbf{w}\rangle\langle\mathbf{w}|\mathbf{z}\rangle\langle\mathbf{z}|) = \langle\mathbf{z}|\mathbf{w}\rangle\langle\mathbf{w}|\mathbf{z}\rangle$$
    where the final equality is derived by taking the trace in an orthonormal basis that includes $|\mathbf{z}\rangle$ as one of its vectors. 
    
    Now recall that the Gaussian RKHS includes Gaussian kernel functions $K_\mathbf{z}(\,\cdot\,)=K(\,\cdot\, ,\mathbf{z})$. The reproducing property means that the inner product between two kernel functions is defined as:
    $$\langle K_\mathbf{w}(\,\cdot\,), K_\mathbf{z}(\,\cdot\,)\rangle =K(\mathbf{w},\mathbf{z}) = \langle\mathbf{z}|\mathbf{w}\rangle\langle\mathbf{w}|\mathbf{z}\rangle .$$
    
    This provides an isometric bijection between the kernel functions $K_\mathbf{z}(\,\cdot\,) = \exp(-|\,\cdot\, - \mathbf{z}|^2)$ (Gaussians centered at $\mathbf{z}$) and the pure coherent states $\hat{\rho}_\mathbf{z} =|\mathbf{z}\rangle\langle\mathbf{z}| $ of the quantum harmonic oscillator. Note that the kernel functions and the pure coherent states are \textit{nonlinear} images of $\mathbb{C}^n$ that sit as submanifolds of their respective (infinite-dimensional) Hilbert spaces. In addition, the Gaussian RKHS is often defined with a scale parameter $\lambda$, i.e.\@ $K(\mathbf{w},\mathbf{z}) = \exp(-|\mathbf{w}-\mathbf{z}|^2/\lambda)$. In the coherent state picture, this parameter can be absorbed by restoring physical dimensions to the phase space coordinates $\mathbf{z}$ and adjusting Planck's constant $\hbar$. Equivalently, one can merely rescale the coordinates as $\mathbf z\mapsto \mathbf z/\sqrt\lambda$ to achieve the same effect.   
\end{proof}

This isometry offers a precise characterization of the functions contained in the Gaussian RKHS over $\mathbb{C}^n$:

\begin{corollary}\label{gaussian-RKHS}
    Any function in the Gaussian RKHS over $\mathbb{C}^n$ can be written uniquely as $f(\mathbf{z}) = \langle\mathbf{z}|\hat{f}|\mathbf{z}\rangle$ for a Hilbert-Schmidt operator $\hat{f}$ on the state space of the $n$-mode quantum harmonic oscillator.
\end{corollary}

\begin{proof}
    In any RKHS, the linear span of the kernel functions is a dense subspace of the full Hilbert space \cite{Muandet2017}. Indeed, the Gaussian RKHS can be constructed from the kernel functions $K_\mathbf{z}(\,\cdot\,)$ as follows. Any function in the linear span of $K_\mathbf{z}$ can be written as:
    $$ f(\,\cdot\,) = \sum_i b_i K_{\mathbf{z}_i}(\,\cdot\,) = \sum_i b_i K(\,\cdot\, , \mathbf{z}_i)$$
    for coefficients $b_i$ and vectors $\mathbf{z}_i \in \mathbb{C}^n$. This expansion is unique because the kernel functions are linearly independent (due to strict positive definiteness of the Gaussian kernel). As the inner product is a sesquilinear form, the inner product between such functions is:
    \begin{align*}
        \langle f(\,\cdot\,),g(\,\cdot\,)\rangle &= \left< \sum_i b_i K_{\mathbf{z}_i}(\,\cdot\,)\;,\; \sum_j c_j K_{\mathbf{z}_j}(\,\cdot\,) \right>   \\
        &=\sum_{i,j} \overline{b_i} c_j\langle K_{\mathbf{z}_i}(\,\cdot\,), K_{\mathbf{z}_j}(\,\cdot\,)\rangle 
                    =  \sum_{i,j} \overline{b_i} c_j K(\mathbf{z}_i, \mathbf{z}_j) .
    \end{align*}
    The reproducing property of the kernel means that:
    \begin{align*}
        f(\mathbf{w}) &= \langle K_\mathbf{w}(\,\cdot\,), f(\,\cdot\,)\rangle = \left< K_\mathbf{w}(\,\cdot\,)\;,\; \sum_i b_i K_{\mathbf{z}_i}(\,\cdot\,) \right> \\
        &= \sum_i b_i\langle K_{\mathbf{w}}(\,\cdot\,), K_{\mathbf{z}_i}(\,\cdot\,)\rangle = \sum_i b_i K(\mathbf{w} , \mathbf{z}_i) = f(\mathbf{w})
    \end{align*}
    as desired. The full Gaussian RKHS is defined as the Cauchy completion of these functions with respect to this inner product.

    Now we apply the isometric map $K_\mathbf{z}(\,\cdot\,) \to \hat{\rho}_\mathbf{z}$. By linear independence, every function $f$ in the linear span of $K_\mathbf{z}(\,\cdot\,)$ is mapped uniquely to an operator $\hat{f}$ in the linear span of $\hat{\rho}_\mathbf{z}$:

    $$ f(\,\cdot\,) = \sum_i b_i K_{\mathbf{z}_i}(\,\cdot\,) \; \to \; \hat{f} = \sum_i b_i \hat{\rho}_{\mathbf{z_i}}$$

    Apply the reproducing property of the kernel, and use the isometry of the inner product under the map $K_\mathbf{z}(\,\cdot\,) \to \hat{\rho}_\mathbf{z}$:
        \begin{align*}
        f(\mathbf{w}) &= \langle K_\mathbf{w}(\,\cdot\,), f(\,\cdot\,)\rangle = \left< K_\mathbf{w}(\,\cdot\,)\;,\; \sum_i b_i K_{\mathbf{z}_i}(\,\cdot\,) \right> \\
        &= \sum_i b_i\langle K_{\mathbf{w}}(\,\cdot\,), K_{\mathbf{z}_i}(\,\cdot\,) \rangle\\
        &= \sum_i b_i \,\mathrm{tr}(\hat{\rho}_\mathbf{w} \hat{\rho}_{\mathbf{z}_i} ) = \mathrm{tr}\left(\hat{\rho}_\mathbf{w} \sum_i b_i\hat{\rho}_{\mathbf{z}_i} \right)  \\
        &= \mathrm{tr}(\hat{\rho}_\mathbf{w} \hat{f} ) = \mathrm{tr}(|\mathbf{w}\rangle\langle\mathbf{w}| \hat{f} ) = \langle\mathbf{w}|\hat{f}|\mathbf{w}\rangle .
    \end{align*}
    Thus, for functions $f$ in the linear span of $K_\mathbf{z}(\,\cdot\,)$, we have a linear isometric map $f \to \hat{f}$ to a Hilbert-Schmidt operator $\hat{f} \in \mathcal{B}_2(\mathcal{H}_F)$. Since the inner product is the same in both spaces by Theorem \ref{isomorphism} (SI Appendix, \ref{isomorphism-proof}), the full Gaussian RKHS, which is  the Cauchy completion of the span of $K_\mathbf{z}(\,\cdot\,)$, is isometrically mapped to the Cauchy completion of the span of $\hat{\rho}_\mathbf{z}$ in $\mathcal{B}_2(\mathcal{H}_F)$.
\end{proof}

It turns out that this map $f \to \hat{f}$ is bijective, which establishes the equivalence of these Hilbert spaces:

\begin{corollary}
     There is an isometric isomorphism between the Hilbert-Schmidt space $\mathcal{B}_2(\mathcal{H}_F)$ and the Gaussian RKHS over $\mathbb{C}^n$.
\end{corollary}

\begin{proof}
    The isometric equivalence between $K_\mathbf{z}(\,\cdot\,)$ and $\hat{\rho}_\mathbf{z}$ ensures that the $\hat{\rho}_\mathbf{z}$ are linearly independent: the Gram matrix is identical, so it must be strictly positive definite. So any operator $\hat{f}$ in the span of $\hat{\rho}_\mathbf{z}$ has a unique preimage $f$ in the span of $K_\mathbf{z}(\,\cdot\,)$. Likewise, since the inner product is the same in both spaces, any $\hat{f}$ in the Cauchy completion of the span of $\hat{\rho}_\mathbf{z}$ has a unique preimage $f$ in the Gaussian RKHS. Thus the map $f \to \hat{f}$ is injective.

    Now we prove surjectivity: every Hilbert-Schmidt operator $\hat{A} \in \mathcal{B}_2(\mathcal{H}_F)$ is the image of some function $A(\,\cdot\,)$ in the Gaussian RKHS. Equivalently, the linear span of the coherent states $\hat{\rho}_\mathbf{z}$ is dense in the entire Hilbert-Schmidt space $\mathcal{B}_2(\mathcal{H}_F)$. Denote the closure of this linear span as $\mathcal{S}$. If $\mathcal{S}$ is a proper closed subspace of $\mathcal{B}_2(\mathcal{H}_F)$, there must be a nonzero operator $\hat{A}$ that is orthogonal to every operator in $\mathcal{S}$; i.e., $\mathrm{tr}(\hat{S} \hat{A} ) = 0$ for $\hat{S} \in \mathcal{{S}}$. Define the function $Q_A(\mathbf{z})  = \langle\mathbf{z}|\hat{A}|\mathbf{z}\rangle$, which is the Hilbert-Schmidt inner product $\mathrm{tr}(\hat{\rho}_\mathbf{z} \hat{A} )$ between $\hat{A}$ and the pure coherent state $\hat{\rho}_\mathbf{z}$. This is known as the $Q$-representation of the operator $\hat{A}$, which is unique for any $\hat{A}$ (see Lemma \ref{Q-unique} below). If $\hat{A} \perp \mathcal{{S}}$, then its $Q$-representation is zero since it is orthogonal to every coherent state in $\mathcal{S}$. By uniqueness of the $Q$-representation, however, this requires that $\hat{A}=0$. So the span of $\hat{\rho}_\mathbf{z}$ is dense in the Hilbert-Schmidt space. We conclude that the Gaussian RKHS over $\mathbb{C}^n$ is isometrically isomorphic to the Hilbert-Schmidt operators on $\mathcal{{H}_F}$.
\end{proof}

\begin{lemma}\label{Q-unique}
    For any bounded operator $\hat{A}$, the $Q$-representation $Q_A(\mathbf{z})  = \langle\mathbf{z}|\hat{A}|\mathbf{z}\rangle$ over the coherent states $|\mathbf{z}\rangle$ is a unique function of $\hat{A}$. Moreover, for self-adjoint $\hat{A}$, $\langle\mathbf{z}|\hat{A}|\mathbf{z}\rangle$ is a bounded real entire function of \textbf{both} variables $\overline{\mathbf{z}}$ and $\mathbf{z}$.
\end{lemma}

\begin{proof}
    Define the function $\tilde{A}(\overline{\mathbf{w}},\mathbf{z})$ as:
    \begin{equation}\label{tilde-func}
        \tilde{A}(\overline{\mathbf{w}},\mathbf{z}) = \frac{\langle\mathbf{w}|\hat{A}|\mathbf{z}\rangle}{\langle\mathbf{w}|\mathbf{z}\rangle}
    \end{equation}
    so the Q-representation is $Q_A(\mathbf{z}) = \tilde{A}(\overline{\mathbf{z}},\mathbf{z})$. Expanding in the Fock basis using multi-index notation (eq.\@ \ref{coherent-def} in Lemma \ref{bargmann}) gives:
    
    \begin{align}\label{sesqui-series}
        \frac{\langle\mathbf{w}|\hat{A}|\mathbf{z}\rangle}{\langle\mathbf{w}|\mathbf{z}\rangle} &= \frac{e^{- \frac{1}{2}(|\mathbf{w}|^2 + |\mathbf{z}|^2)}\sum_{\mathbf{n},\mathbf{m}\geq\mathbf{0}}(\mathbf{n}!\mathbf{m}!)^{-1/2}\, \overline{\mathbf{w}}^\mathbf{n}\mathbf{z}^\mathbf{m} \langle\mathbf{n}|\hat{A}|\mathbf{m}\rangle}{e^{- \frac{1}{2}(|\mathbf{w}|^2 + |\mathbf{z}|^2 - 2 \overline{\mathbf{w}}\mathbf{z}) }} \notag \\
        &= e^{-\overline{\mathbf{w}}\mathbf{z}}\sum_{\mathbf{n},\mathbf{m}=\mathbf{0}}^\mathbf{\infty}\frac{\langle\mathbf{n}|\hat{A}|\mathbf{m}\rangle}{\sqrt{\mathbf{n}!\mathbf{m}!}}\, \overline{\mathbf{w}}^\mathbf{n}\mathbf{z}^\mathbf{m} 
    \end{align}
    which converges for all $\overline{\mathbf{w}},\mathbf{z}$ by boundedness of $\langle\mathbf{n}|\hat{A}|\mathbf{m}\rangle \leq \|\hat{A}\| $ and rapid decrease of $(\mathbf{n}!\mathbf{m}!)^{-1/2}$. Thus $\tilde{A}(\overline{\mathbf{w}},\mathbf{z})$ is a sesquiholomorphic entire function (holomorphic in $\mathbf{z}$, anti-holomorphic in $\mathbf{w}$) whose series coefficients are defined by the Fock-basis matrix elements of $\hat{A}$. Since the series expansion of an entire function is unique, and a bounded operator is uniquely determined by its matrix elements in a complete basis, the map $\hat{A} \to \tilde{A}(\overline{\mathbf{w}},\mathbf{z})$ is injective.

    The $Q$-representation is:
    \begin{equation}\label{Q-rep}
         Q_A(\mathbf{z}) = \tilde{A}(\overline{\mathbf{z}},\mathbf{z}) = e^{-\overline{\mathbf{z}}\mathbf{z}}\sum_{\mathbf{n},\mathbf{m}=\mathbf{0}}^\mathbf{\infty}\frac{\langle\mathbf{n}|\hat{A}|\mathbf{m}\rangle}{\sqrt{\mathbf{n}!\mathbf{m}!}}\, \overline{\mathbf{z}}^\mathbf{n}\mathbf{z}^\mathbf{m}
    \end{equation}
    which is an entire function of $\overline{\mathbf{z}},\mathbf{z}$ (treated as independent variables), and bounded since $\hat{A}$ is bounded. The map $\hat{A} \to \tilde{A}(\overline{\mathbf{z}},\mathbf{z})$ is injective by the same argument as above: uniqueness of series coefficients for an entire function, and uniqueness of matrix elements in a complete basis. Thus the $Q$-representation is unique. Finally, if $\hat{A}$ is self-adjoint, its Q-representation $\langle\mathbf{z}|\hat{A}|\mathbf{z}\rangle$ must be real by the Hermitian symmetry of diagonal matrix elements.
\end{proof}

\subsubsection{Theorem \ref{completeness}}\label{completeness-proof}
\emph{Every QPM completely metrizes $\mathcal{P}_\mathrm{w}(X)$.}
\begin{proof}
    We assume that we have a QPM on $\mathcal{P}_\mathrm{w}(X)$. As defined in the main text, this requires a function $\phi : X \to \mathcal{S}$ that maps points $x \in X$ to pure states $|x\rangle\langle x| \equiv \hat{\rho}_x$ on a Hilbert space. We require that $\phi$ is a closed embedding, so its image $\mathcal{C}$ is closed and homeomorphic to $X$, and also that its kernel $k(x,y) = \mathrm{tr}(\hat{\rho}_x\hat{\rho}_y)$ is characteristic. All of these properties are preserved if we normalize the states $\langle x|x\rangle=\mathrm{tr}(\hat{\rho}_x)=1$, so we assume normalized states without loss of generality; this gives a normalized kernel $k(x,x)=1$.

    The barycenter map is:
    $$ T(\mu) = \int_X \hat{\rho}_x\, d\mu(x) \equiv \hat{\mu} $$
    which is well-defined as a strong (Bochner) integral since the map $x\to \hat{\rho}_x$ is continuous from $X$ to the separable space $\mathcal{B}_1(\mathcal{H})$, and the trace norm $\|\hat{\rho}_x\|_1=1$. Since $\mu$ is a probability measure, it is straightforward to show that $\hat{\mu}$ is an element of the quantum state space $\mathcal{S}$: it is a positive operator with unit trace. By definition, the barycenter map is injective for a characteristic kernel.

    By assumption, the image $\phi(X) \equiv\mathcal{C}$ is a (norm-)closed subset of $\mathcal{S}$ that is homeomorphic to $X$. Labeling operators in $\mathcal{C}$ by their corresponding points in $X$ and viewing $\mu$ as the pushforward measure on $\mathcal{C}$, we can write the barycenter map as an integral on $\mathcal{C}$:
    $$ T(\mu) = \int_\mathcal{C} \hat{\rho}_x\, d\mu(x) \equiv \hat{\mu} $$
    which is now viewed as a measure on the closed set $\mathcal{C} \subset \mathrm{extr}\,\mathcal{S}$ (recall that pure states are the extreme points of the convex quantum state space $\mathcal{S}$). Measures on $\mathcal{C}$ inherit the weak topology from measures on $X$, and vice versa, via the homeomorphism $\phi$. The barycenter map $T(\mu)$ is continuous from $\mathcal{P}_\mathrm{w}(\mathcal{C})$ to $\mathcal{S}$ with the trace norm \cite{Shirokov2007}, it is convex (affine), and its image is contained in the closed convex hull of $\mathcal{C}$, denoted $\overline{\mathrm{co}}(\mathcal{C})$.

    Proposition 4 from \cite{Shirokov2007} states that the quantum state space $\mathcal{S}$ is a \textit{$\mu$-compact set}, which means that  the barycenter map on $\mathcal{S}$ is a \textit{proper map} (the preimage of a compact set is compact). This property extends to closed subsets such as $\mathcal{C}$: the preimage of any compact subset of $\overline{\mathrm{co}}(\mathcal{C})$ is a compact subset of $\mathcal{P}_\mathrm{w}(\mathcal{C})$. The key results of Choquet theory, which usually apply only to compact sets, extend to $\mu$-compact sets. For instance, Proposition 2 of \cite{Shirokov2007} shows that the barycenter map $T(\mu)$ is surjective onto $\overline{\mathrm{co}}(\mathcal{C})$: every element of $\overline{\mathrm{co}}(\mathcal{C})$ is the barycenter of some measure on $\mathcal{C}$. Since the barycenter map is injective (by assumption that the kernel $k(x,y) = \mathrm{tr}(\hat{\rho}_x\hat{\rho}_y)$ is characteristic), these measures are unique. Thus we have a continuous bijection from measures on $\mathcal{C}$ to quantum states in the closed convex set $\overline{\mathrm{co}}(\mathcal{C})$.

    To show that the barycenter map has a continuous inverse, note that any continuous proper map $f: X \to Y$ between metrizable spaces is closed: it maps closed sets to closed sets. (If $U\subseteq X$ is closed but $f(U)$ is not, pick $y_n \in f(U)$ converging to $y \notin f(U)$. The set $\{y_n\} \cup \{y\}$ is compact, so by properness, $\{x_n : f(x_n)=y_n\}$ has a convergent subsequence $x_{n_k} \to x \in U$. Continuity gives $f(x) = y$, contradicting $y \notin f(U)$. Thus $f(U)$ must be closed.) Since any continuous closed bijection is a homeomorphism, the barycenter map is a homeomorphism between the metrizable spaces $\mathcal{P}_\mathrm{w}(\mathcal{C})$ and $\overline{\mathrm{co}}(\mathcal{C})$. Note that this construction works for \textit{any} characteristic kernel $k(x,y) = \mathrm{tr}(\hat{\rho}_x\hat{\rho}_y)$; homeomorphism follows from the topological properties of the barycenter map on $\mathcal{S}$.

    Finally, since the embedding map $\phi$ is a homeomorphism from $X$ onto its image $\mathcal{C}$, the spaces $\mathcal{P}_\mathrm{w}(\mathcal{C})$ and $\mathcal{P}_\mathrm{w}(X)$ are likewise homeomorphic. Thus the barycenter map establishes a homeomorphism between $\mathcal{P}_\mathrm{w}(X)$ and $\overline{\mathrm{co}}(\mathcal{C})$. Since this is a closed set in the complete metric space $\mathcal{B}_1(\mathcal{H})$, its metric (the trace distance) completely metrizes $\mathcal{P}_\mathrm{w}(X)$.
\end{proof}

As a characteristic kernel is essential for this construction, the following two lemmas highlight important features of such kernels:

\begin{lemma}\label{characteristic}
    For any normalized kernel $k(x,y) = \mathrm{tr}(\hat{\rho}_x\hat{\rho}_y)$, $k$ is characteristic if and only if it is integrally strictly positive definite, i.e.\@ for all finite complex measures $m \neq0$ we have:
        $$ \iint_{X \times X} k(x,y)\,d\overline{m}(x)\,dm(y) > 0$$
\end{lemma}
\begin{proof}
    First we show that being integrally s.p.d.\ is equivalent to having an injective barycenter map for all complex measures. Suppose $\iint k(x,y)\,d\overline{m}(x)\,dm(y) = 0$ for some nonzero $m$. Choose an arbitrary pair of distinct measures $m_1, m_2$ with $m_1-m_2 = m$ and calculate the Hilbert-Schmidt distance between the barycenter embeddings of these measures (Bochner integrals):
    \begin{align*}
        \mathrm{tr}\big(|T(m_1)-T(m_2)|^2\big) &= \mathrm{tr}\left(\left| \int_X   \hat{\rho}_x\, dm_1(x) - \!\!\int_X   \hat{\rho}_x\, dm_2(x)\right|^2\right) \\
          &= \mathrm{tr}\left(\left| \int_X   \hat{\rho}_x\, dm(x)\right|^2\right) \\
          &= \mathrm{tr}\left( \int_X   \hat{\rho}_x\, d\overline{m}(x) \int_X   \hat{\rho}_y\, dm(y)\right) \\
          &= \iint_{X\times X}  \mathrm{tr}(\hat{\rho}_x\hat{\rho}_y)\, d\overline{m}(x)\, dm(y) = 0 
    \end{align*}
    so the barycenter map is not injective for $m_1, m_2$. (The trace and integral can be interchanged using Fubini-Tonelli by boundedness of $|\mathrm{tr}(\hat{\rho}_x\hat{\rho}_y)|$ and finite total variation of $|m|$). Conversely, if the barycenter map is not injective for some pair $m_1, m_2$, then the same calculation shows that for $m = m_1-m_2$ we must have $\iint k(x,y)\,d\overline{m}(x)\,dm(y) = 0$, so the kernel is not integrally s.p.d.

    Now we prove the lemma. Since the barycenter map with an integrally s.p.d.\ kernel is injective for all complex measures, it is injective for probability measures. Thus  integrally s.p.d.\@ implies characteristic. Although in general the converse does not hold \cite{Muandet2017}, it does hold for kernels of the form $k(x,y) = \mathrm{tr}(\hat{\rho}_x\hat{\rho}_y)$. Suppose we have two distinct complex measures $m_1, m_2$. Since the states are normalized, $\langle x|x\rangle=\mathrm{tr}(\hat{\rho}_x)=1$, the trace of the barycenter embedding of any measure is equal to its total measure $m(X)$:
    $$ \mathrm{tr}(T(m)) = \int_X \mathrm{tr}(\hat{\rho}_x)\,dm(x) = \int_X \,dm(x) = m(X)$$
    So the embedded measures will be distinct (different traces) unless they have the same total measure, $m_1(X) = m_2(X)$. Note that this argument is not available in the usual RKHS context: the trace of the barycenter embedding map is equivalent to integration against the constant witness function $\mathbf{1}$, which may not belong to the RKHS. As pointed out in the main text, for many commonly-used kernels such as the Gaussian or Laplacian kernels, all RKHS functions vanish at infinity so constant functions are not included. Using the trace embedding, on the other hand, the identity operator maps to the constant function $\mathbf{1}$, which is always included as a dual function (witness function).
    
    If $m_1(X) = m_2(X)$, define $m = m_1-m_2$ so that $m(X)=0$. Assume that the barycenter map is not injective for $m_1, m_2$, i.e.\@ $T(m_1)-T(m_2)= T(m)=0$. The complex measure $m$ has Jordan decomposition $m = (\mu_1-\mu_2) + i(\mu_3-\mu_4)$, and since $m(X)=0$, these positive measures must have total measure $\mu_1(X)=\mu_2(X)$ and $\mu_3(X)=\mu_4(X)$. Unless they vanish identically, each pair can be normalized to yield probability measures for which the barycenter map is injective. Thus, $T(\mu_1)-T(\mu_2)= T(\mu_1 - \mu_2)=0$ forces $\mu_1=\mu_2$, and likewise for $\mu_3=\mu_4$, which gives $m = m_1-m_2 = 0$, contradicting the assumption that $m_1, m_2$ are distinct. So the barycenter map is indeed injective for all complex measures, which means that the kernel $k(x,y) = \mathrm{tr}(\hat{\rho}_x\hat{\rho}_y)$ is integrally s.p.d.
\end{proof}

The next lemma gives a simple construction for kernels of the desired form starting with any integrally s.p.d.\ kernel $k'(x,y) = \langle x|y\rangle$ on the underlying Hilbert space: simply use the norm-squared kernel $k = |k'|^2 = \langle y|x \rangle\langle x|y \rangle = \mathrm{tr}(\hat{\rho}_x\hat{\rho}_y)$. We show that the resulting kernel is integrally strictly positive definite:

\begin{lemma}\label{schur-product}
    If the kernel $k'(x,y) = \langle x|y\rangle$ is integrally strictly positive definite, then so is the norm-squared kernel $k = |k'|^2 = \langle y|x \rangle\langle x|y \rangle = \mathrm{tr}(\hat{\rho}_x\hat{\rho}_y)$.
\end{lemma}

\begin{proof}
    On the dual $\mathcal{H}^*$ of the underlying Hilbert space $\mathcal{H}$, the barycenter map $B(m)$ for a (generally complex) measure $m$ is:
    $$B(m) = \int_X \langle x|\,dm \equiv \langle m|$$
    where $\langle m|$ represents the embedded dual vector corresponding to the measure $m$. Embedding as dual vector $\langle m|\in\mathcal{H}^*$ instead of as a vector $|m\rangle$ is conventional because it allows the expected value of a function $|f\rangle$ in the RKHS to be calculated as $\langle m|f\rangle$.
    
    Following the same argument in Lemma \ref{characteristic}, since $k'(x,y) = \langle x|y\rangle$ is integrally s.p.d., $B$ is injective \cite{Muandet2017}. Thus, for any two distinct measures $m_1,m_2$ there is a witness vector $|w\rangle$ such that $\langle m_1|w\rangle \neq \langle m_2|w\rangle$. Since $\langle x|w\rangle$ is a continuous function of $x$, define new complex measures by setting $dM_1 \equiv \langle x|w\rangle dm_1$ and $dM_2 \equiv \langle x|w\rangle dm_2$ and note that they have different total measures:
    $$ \int_X dM_1 = \int_X  \langle x|w\rangle dm_1\, \neq \int_X \langle x|w\rangle dm_2 \, =\int_X dM_2$$
    so $M_1 \neq M_2$ (and both total measures are finite).

    Now assume for contradiction that the norm-squared kernel $k =|k'|^2 = \mathrm{tr}(\hat{\rho}_x\hat{\rho}_y)$ is not integrally s.p.d. By Lemma \ref{characteristic}, this means that its barycenter map
    $$ T(m) = \int_X \hat{\rho}_x\,dm(x) = \int_X |x\rangle\langle x|\,dm(x)$$
    fails to be injective, i.e.\@ for some distinct measures $m_1,m_2$ we have $T(m_1)=T(m_2)$. Following the argument above and applying the operator $T(m_1) - T(m_2)$ to the witness vector $|w\rangle$ gives:
    \begin{align*}
        0 &= \int_X |x\rangle\langle x|w\rangle\,dm_1(x) - \int_X |x\rangle\langle x|w\rangle\,dm_2(x) \\
        &= \int_X |x\rangle\,dM_1(x) - \int_X |x\rangle\,dM_2(x)
    \end{align*}
    Taking the dual, we have:
    $$0 = \int_X \langle x|\,d\overline{M_1}(x) - \int_X \langle x|\,d\overline{M_2}(x) = B(\overline{M_1}) - B(\overline{M_2}) $$
    Since $B$ is injective, this requires $M_1 = M_2$, which contradicts our earlier conclusion that $M_1 \neq M_2$. Hence $k$ must be integrally strictly positive definite.
\end{proof}
    Note that although the norm-squared kernel inherits the property of being integrally s.p.d.\ from its base kernel, the same is not true for being characteristic. As a counterexample, consider a two-element space $X = \{a,b\}$ and define the kernel
    $$ k'(x,y) = 
    1,  \text{ if } x=y, \qquad
    k'(x,y) =-1, \text{ if } x \neq y.
    $$
    For a probability measure $P = (P_a,P_b)$ with $P_a + P_b = 1$, the embedding induced by $k'$ is:
    $$B(P)(x) = \sum_{y\in\{a,b\}}k'(x,y)P_y\;\;\mathrm{for}\;x \in \{a,b\}.$$
    In particular,
    $$B(P)(a) = k'(a,a)P_a + k'(a,b)P_b = P_a-P_b$$
    $$B(P)(b) = k'(b,a)P_a + k'(b,b)P_b = P_b-P_a$$
    Thus, the embedding can be written as $B(P) = (P_a-P_b,P_b-P_a)$. Since the difference $P_a - P_b$ uniquely characterizes the measure, this embedding is injective and hence $k'$ is characteristic.

    Now, consider the squared kernel $k(x,y)=|k'(x,y)|^2$ which has all entries equal to 1. Its embedding maps any probability measure $P = (P_a,P_b)$ to:
    $$B(P) = (P_a + P_b, P_b + P_a) = (1,1).$$
    Since all measures are mapped to the same point, the squared kernel is not characteristic. In this case, neither kernel is injective on the entire space of finite measures (as opposed to probability measures). For instance, the kernel $k'$ is not injective---it gives the same embedding for measures $(1,1)$ and $(2,2)$.

    Finally, note that Lemma \ref{schur-product} can be viewed as a continuous, integral generalization of the Schur product theorem. Recall that this theorem states that the entrywise Hadamard product of two (strictly) positive definite matrices is (strictly) positive definite. As a special case, if a matrix $A_{ij}$ is s.p.d.\@ ($\langle v|A|v\rangle > 0$ for all nonzero $v$), then its Hadamard product with its own conjugate, $B_{ij} = |A_{ij}|^2$, is also s.p.d. In our setting, integral strict positive definiteness corresponds to positivity under integration against nonzero complex measures, rather than positivity under summation. Thus, the claim presented here is a generalization of this special case of the Schur product theorem, replacing finite sums with integrals and discrete vectors with measures. Indeed, one can extend this proof to the general case of two integrally s.p.d.\@ kernels $k_1, k_2$ and show that their product $k_1k_2$ is also integrally s.p.d.

\subsubsection{Theorem \ref{universality}}\label{universality-proof}
\emph{Every Polish space has a QPM.}
\begin{proof}
    Given a Polish space $X$---a separable, completely metrizable space---we seek a map $\phi : X \to \mathcal{S} $ that maps points $x \in X$ to pure states $|x\rangle\langle x| \equiv \hat{\rho}_x$ on a Hilbert space. We require $\phi$ to be a closed embedding with a characteristic kernel $k(x,y) = \mathrm{tr}(\hat{\rho}_x\hat{\rho}_y)$.

    First, choose a metric $d$ that completely metrizes $X$. Embed $(X, d)$ isometrically as a closed subset of the Urysohn universal metric space, which is homeomorphic to a separable Hilbert space $\mathcal{H}$ \cite{Uspenskij2004}. This closed embedding maps  $x \in X$ to vectors $\mathbf{x} \in \mathcal{H}$.

    Following the construction in Lemma \ref{bargmann}, extended to infinite dimensions as in Section 10 of \cite{Hall2000}, build the Fock space $\mathcal{H}_F$ over $\mathcal{H}$. Each vector $\mathbf{x}$ in the original Hilbert space is mapped to a coherent state vector $|\mathbf{x}\rangle$ in $\mathcal{H}_F$, with inner product:
    $$ \langle\mathbf{x}|\mathbf{y}\rangle = e^{- \frac{1}{2}(|\mathbf{x}|^2 + |\mathbf{y}|^2 - 2 (\mathbf{x},\mathbf{y}) ) }$$
    where $(\mathbf{x},\mathbf{y})$ is the inner product in $\mathcal{H}$.

    Define the map $\phi$ in terms of these coherent states as:
    $$ \phi : x \to |\mathbf{x}\rangle\langle\mathbf{x}| \equiv \hat{\rho}_x \in \mathcal{S}$$
    where $\mathcal{S}$ is the quantum state space over $\mathcal{H}_F$; this gives a closed embedding of $X$ as pure states in $\mathcal{S}$. The kernel is:
    $$ k(x,y) = \mathrm{tr}(\hat{\rho}_x\hat{\rho}_y) = \langle\mathbf{x}|\mathbf{y}\rangle\langle\mathbf{y}|\mathbf{x}\rangle = e^{- |\mathbf{x} - \mathbf{y}|^2 }$$
    which is the Gaussian kernel for embedded vectors $\mathbf{x},\mathbf{y} \in \mathcal{H}$. The Gaussian kernel on a separable Hilbert space is characteristic (Theorem 3.1 of \cite{ziegel2022}). Thus the map $\phi$ has all of the properties required for Theorem \ref{completeness} (SI Appendix \ref{completeness-proof}): it defines a QPM that completely metrizes the Polish space $X$. If one already has a closed embedding of $X$ into a vector space (as is often the case in machine learning), the Gaussian kernel---or, indeed, any integrally s.p.d.\@ kernel---immediately gives a QPM for this space.
\end{proof}

\subsubsection{Theorem \ref{uniform-approx}}\label{uniform-approx-proof}
\emph{Dual functions for the Fock (coherent state) QPM are dense in $BUC(\mathbb{R}^n)$: they can \textbf{uniformly} approximate any bounded, uniformly continuous function on $\mathbb{R}^{n}$.}

This proof is somewhat long and technical, so we first offer a short overview. For reasons that will be clear later, we consider functions on $\mathbb{R}^{2n}$, represented as $F(\mathbf{x},\mathbf{y})$ for $\mathbf{x},\mathbf{y} \in \mathbb{R}^n$; any function on $\mathbb{R}^m$ can trivially be embedded in $\mathbb{R}^{2n}$ for $2n\geq m$. The proof proceeds as follows.
\begin{enumerate}
    \item We start by considering real functions $F(\mathbf{x},\mathbf{y})$ that are bounded on $\mathbb{R}^{2n}$ and band-limited (i.e.\@ distributional Fourier transform has compact support). Using a Paley-Wiener type result (Lemma \ref{paley-wiener} below), such a function extends to an entire function of exponential type on $\mathbb{C}^{2n}$, with exponential growth depending only on the imaginary parts of its arguments. 
    \item Using the Q-representation (eq.\@ \ref{Q-rep} in Lemma \ref{Q-unique}), this (real) entire function $F(\mathbf{x},\mathbf{y})$ determines the unique Fock-basis matrix elements of a Hermitian operator $\hat{F}$ on the Fock space $\mathcal{H}_F$, with $F(\mathbf{x},\mathbf{y})=Q_F(\mathbf{x}+i\mathbf{y})=\langle\mathbf{x}+i\mathbf{y}|\hat{F}|\mathbf{x}+i\mathbf{y}\rangle$ (these are coherent-state vectors for complex $\mathbf{z}=\mathbf{x}+i\mathbf{y}$, with  $\mathbf{x},\mathbf{y} \in \mathbb{R}^n$).
    \item Using the Segal-Bargmann-Fock correspondence (Lemma \ref{bargmann}), the operator $\hat{F}$ on $\mathcal{H}_F$ is isometrically isomorphic to an integral operator on $L^2(\mathbb{C}^n)$ whose integral kernel is bounded by the exponential growth bounds for $F$.
    \item Using the Schur test \cite{Grafakos2014}, this integral kernel defines a \textit{bounded} operator on $L^2(\mathbb{C}^n)$, which implies that the operator $\hat{F}$ is indeed a bounded, self-adjoint operator on the Fock space $\mathcal{H}_F$. Proving boundedness is the key step in this theorem.
    \item Since the bounded self-adjoint operators are dual to the self-adjoint trace-class operators, every bounded, band-limited function is a dual function, with $F(\mathbf{x},\mathbf{y})=\langle\mathbf{x}+i\mathbf{y}|\hat{F}|\mathbf{x}+i\mathbf{y}\rangle$.
    \item Finally, using the results of \cite{Dryanov2003}, these dual functions can uniformly approximate any bounded, uniformly continuous function on $\mathbb{R}^{2n}$.
\end{enumerate}

We begin with a technical lemma of Paley-Wiener type:

\begin{lemma}\label{paley-wiener}
    Any band-limited real function $f(\mathbf{x})$ with $|f(\mathbf{x})|\leq M$ on $\mathbb{R}^n$ extends to an entire function $f(\mathbf{z})$ on $\mathbb{C}^n$ with:
    $$ |f(\mathbf{z})| \leq M \exp(\tau_1|\Im z_1|+\tau_2|\Im z_2|+\ldots+\tau_n|\Im z_n|)$$
    for positive constants $\tau_j$.
\end{lemma}

\begin{proof}
    Begin with the one-dimensional case. By the Paley-Wiener-Schwartz theorem, any real function $f(x)$ that is band-limited (distributional Fourier transform supported in $[-\tau, \tau]$) extends to an entire function of exponential type $\tau$. Then, by Proposition 1.1 in \cite{Dryanov2003}, if in addition $|f(x)|\leq M$ on the real axis, then for all $z \in \mathbb{C}$ we have $|f(z)| \leq M e^{\tau|\Im z|}$. (We adopt the Fourier convention $ \widehat f(\xi)=\int e^{-i x\xi}f(x)\,dx$.)

    For the multivariate case, if a function $f(\mathbf{x})$ on $\mathbb{R}^n$ has a compactly-supported distributional Fourier transform, it extends to an entire function $f(\mathbf{z})$ of exponential type on $\mathbb{C}^n$, i.e., there are positive constants $A, B$ with $|f(\mathbf{z})| \leq Ae^{B|\mathbf{z}|}$. Such a function is of exponential type in each complex coordinate alone, treating the other coordinates as fixed. That is, if $f(\mathbf{z})$ is of exponential type, then for $z_k$ alone we have:
    $$|f(z_1, \ldots, z_k,\ldots,z_n)| \leq A e^{B\left(\sqrt{|z_1|^2 + \ldots + |z_k^2| + \ldots + |z_n|^2}\right)}$$
    $$\leq A e^{B\left(\sqrt{|z_1|^2 + \ldots +|z_{k-1}|^2 + |z_{k+1}|^2 + \ldots + |z_n|^2} + |z_k|\right)}$$
    $$\leq A'\exp(B|z_k|)$$
    for a new constant $A'$ that depends on all coordinates except $z_k$, and these coordinates are held fixed. Thus $f(z_1, \ldots, z_k,\ldots,z_n)$, treated as a single-variable function of $z_k$, is of exponential type in that variable alone.

    Now, a function $f(\mathbf{x})$ as described in the lemma extends to an entire function of exponential type $f(\mathbf{z})$. Write each complex coordinate as $z_j = x_j+iy_j$. Since $f$ is bounded on $\mathbb{R}^n$, we have:
    $$ |f(x_1, x_2, x_3, \ldots, x_n)| \leq M .$$
    Treating $f$ as a function of the first coordinate alone, with $x_2,\ldots,x_n$ fixed, we have a function of exponential type that is bounded by $M$ on the real axis (i.e.\@ for all $x_1 \in \mathbb{R}$). Applying the single-variable result above, and extending the first coordinate to $x_1+iy_1 \in \mathbb{C}$, we have, for some constant $\tau_1$,
    $$ |f(x_1 + iy_1, x_2, x_3, \ldots,x_n)| \leq M\exp(\tau_1|y_1|).$$
    
    Now treat $f$ as a function of the second coordinate, with $x_1 + iy_1$ and $x_3,\ldots,x_n$ fixed; we have a function of exponential type that is bounded for all $x_2 \in \mathbb{R}$ by $M\exp(\tau_1|y_1|)$. Extending the second coordinate to $\mathbb{C}$ and applying the single-variable result gives:
    $$ |f(x_1 + iy_1, x_2 + iy_2, x_3, \ldots,x_n)| \leq M\exp(\tau_1|y_1|+\tau_2|y_2|)$$
    for some constant $\tau_2$. Iterating in this fashion proves the lemma. With additional effort (not needed for this proof), one can show that the constants $\tau_j$ define the support of the distributional Fourier transform as the box $[-\tau_1,\tau_1] \times [-\tau_2,\tau_2]\times\ldots\times[-\tau_n,\tau_n]$.
\end{proof}

We need one more lemma that connects the analytic continuation of $F(\mathbf{x},\mathbf{y})$ on $\mathbb{C}^{2n}$ with the normalized off-diagonal coherent-state matrix elements $\langle\mathbf{w}|\hat{F}|\mathbf{z}\rangle$ of the corresponding operator $\hat{F}$:

\begin{lemma}\label{conjugate-off-diag}
    Given a real entire function $F$ of exponential type, with $F(\mathbf{x},\mathbf{y})=Q_F(\mathbf{x}+i\mathbf{y})=\langle\mathbf{x}+i\mathbf{y}|\hat{F}|\mathbf{x}+i\mathbf{y}\rangle$ for a Hermitian operator $\hat{F}$, the extension of $F$ to complex-valued $\mathbf{x},\mathbf{y}$ is given by:
    \begin{equation}\label{off-diag-extension}
        F(\mathbf{x},\mathbf{y}) = \frac{\langle\overline{\mathbf{x}}+i\overline{\mathbf{y}}|\hat{F}|\mathbf{x}+i\mathbf{y}\rangle}{\langle\overline{\mathbf{x}}+i\overline{\mathbf{y}}|\mathbf{x}+i\mathbf{y}\rangle}
    \end{equation}
    where $\overline{\mathbf{x}},\overline{\mathbf{y}}$ represent the complex conjugates of $\mathbf{x},\mathbf{y} \in \mathbb{C}^n$.
\end{lemma}

\begin{proof}
    Recall from eq.\@ \ref{Q-rep} of Lemma \ref{Q-unique} that the Q-representation of the operator $\hat{F}$ has the absolutely convergent power series:
    \begin{equation}\label{later-Q-def}
        Q_F(\mathbf{z}) = \langle\mathbf{z}|\hat{F}|\mathbf{z}\rangle = e^{-\overline{\mathbf{z}}\mathbf{z}}\sum_{\mathbf{n},\mathbf{m}=\mathbf{0}}^\mathbf{\infty}\frac{\langle\mathbf{n}|\hat{F}|\mathbf{m}\rangle}{\sqrt{\mathbf{n}!\mathbf{m}!}}\, \overline{\mathbf{z}}^\mathbf{n}\mathbf{z}^\mathbf{m}
    \end{equation}
    From the usual perspective of complex analyticity, this is a peculiar power series because it is entire for both $\mathbf{z}$ and $\overline{\mathbf{z}}$ \textit{independently}. In some ways, this expression looks more natural if we express it in terms of the independent real and imaginary parts of $\mathbf{z}=\mathbf{x}+i\mathbf{y}$:
    $$ Q_F(\mathbf{x}+i\mathbf{y}) = e^{-(\mathbf{x}^2+\mathbf{y}^2)}\!\sum_{\mathbf{n},\mathbf{m}=\mathbf{0}}^\mathbf{\infty}\frac{\langle\mathbf{n}|\hat{F}|\mathbf{m}\rangle}{\sqrt{\mathbf{n}!\mathbf{m}!}}(\mathbf{x}-i\mathbf{y})^\mathbf{n}(\mathbf{x}+i\mathbf{y})^\mathbf{m}$$
    (we notate $\mathbf{x}^2 = \mathbf{x}\cdot\mathbf{x} = x_1^2 + \ldots + x_n^2$, and same for $\mathbf{y}^2$; these are holomorphic squares not norm-squares.) Written in this form, it is evident that the right-hand side is a real analytic function of $\mathbf{x}, \mathbf{y}$ that extends to an entire function for \textit{complex} $\mathbf{x}, \mathbf{y}$. But the left-hand side cannot be extended to complex arguments, because the expression $\mathbf{x}+i\mathbf{y}$ for complex $\mathbf{x}, \mathbf{y}$ does not capture the $4n$ real degrees of freedom contained in the two complex vectors $\mathbf{x}, \mathbf{y} \in \mathbb{C}^n$.

    Consider instead the normalized off-diagonal matrix elements defined in eq.\@ \ref{tilde-func} of Lemma \ref{Q-unique} as:
    $$\tilde{F}(\overline{\mathbf{w}},\mathbf{z}) = \frac{\langle\mathbf{w}|\hat{F}|\mathbf{z}\rangle}{\langle\mathbf{w}|\mathbf{z}\rangle}$$
    where we define the invertible map $\mathbf{w} = \overline{\mathbf{x}}+i\overline{\mathbf{y}},\;\;\mathbf{z} = \mathbf{x}+i\mathbf{y}$. Using eq.\@ \ref{sesqui-series} from Lemma \ref{Q-unique}, we obtain (note $\overline{\mathbf{w}} = \mathbf{x}-i\mathbf{y}$; we use multi-index notation $\mathbf{n}!=n_1!n_2!\ldots n_n!$ and $\mathbf{z}^\mathbf{n}=z_1^{n_1}z_2^{n_2}\ldots z_n^{n_n}$):
    \begin{align}\label{wz-to-xy}
        \frac{\langle\mathbf{w}|\hat{F}|\mathbf{z}\rangle}{\langle\mathbf{w}|\mathbf{z}\rangle} 
        &= e^{-\overline{\mathbf{w}}\mathbf{z}}\sum_{\mathbf{n},\mathbf{m}\geq\mathbf{0}}\frac{\langle\mathbf{n}|\hat{F}|\mathbf{m}\rangle}{\sqrt{\mathbf{n}!\mathbf{m}!}}\, \overline{\mathbf{w}}^\mathbf{n}\mathbf{z}^\mathbf{m} \notag \\
        \frac{\langle\overline{\mathbf{x}}+i\overline{\mathbf{y}}|\hat{F}|\mathbf{x}+i\mathbf{y}\rangle}{\langle\overline{\mathbf{x}}+i\overline{\mathbf{y}}|\mathbf{x}+i\mathbf{y}\rangle} 
        &= e^{-(\mathbf{x}^2+\mathbf{y}^2)}\sum_{\mathbf{n},\mathbf{m}}\!\frac{\langle\mathbf{n}|\hat{F}|\mathbf{m}\rangle}{\sqrt{\mathbf{n}!\mathbf{m}!}} (\mathbf{x}\!-\!i\mathbf{y})^\mathbf{n}(\mathbf{x}\!+\!i\mathbf{y})^\mathbf{m} \notag \\
        &= F(\mathbf{x},\mathbf{y})\;\;\mathrm{for}\;\mathbf{x}, \mathbf{y} \in \mathbb{C}^n   \\
        &= Q_F(\mathbf{x}+i\mathbf{y})\;\;\mathrm{for}\;\mathbf{x}, \mathbf{y} \in \mathbb{R}^n \notag
    \end{align}
    Note that the values of the entire function $F(\mathbf{x},\mathbf{y})$ for $\mathbf{x}, \mathbf{y} \in \mathbb{C}^n$ are determined by its values on the real subspace. Likewise, the matrix elements $\langle\mathbf{w}|\hat{F}|\mathbf{z}\rangle$ for any pair of coherent states are determined by their values on the diagonal subspace $\mathbf{w}=\mathbf{z}$. Moreover, since the coherent states are a total set that span the Fock space $\mathcal{H}_F$, these matrix elements completely define the operator $\hat{F}$.
\end{proof}

We now prove the main theorem.

\begin{proof} (Theorem \ref{uniform-approx}, SI Appendix, \ref{uniform-approx-proof})
    First, we show that, for the Fock embedding, the dual functions include any function $F(\mathbf{x},\mathbf{y})$ that is bounded on $\mathbb{R}^{2n}$ and band-limited (i.e.\@ its distributional Fourier transform has compact support). We start with $n=1$ for simplicity.

    Suppose we are given a real, band-limited function that is bounded, $F(x,y)\leq M$ for $x,y \in \mathbb{R}$. We seek an operator $\hat{F}$ on the Fock space $\mathcal{H}_F$ such that, for $z=x+iy$, we have:
    $$ tr(\hat{\rho}_z\hat{F}) = \langle z|\hat{F}|z\rangle = \langle x+iy|\hat{F}|x+iy\rangle=F(x,y). $$
    If such an operator is bounded and self-adjoint, then it defines a dual function as desired. Using Lemma \ref{paley-wiener}, we know that $F$ is an entire function with an everywhere-convergent power series in $x,y$. The product of two entire functions is entire, so the product $F(x,y)e^{(x^2+y^2)}$ has an everywhere-convergent power series with unique coefficients $A_{nm}$:
    $$ F(x,y)e^{(x^2+y^2)}=\sum_{n,m\geq0}A_{nm}x^ny^m$$
    Defining $z=x+iy$ and substituting $x=(z+\overline{z})/2$, $y=(z-\overline{z})/(2i)$ we get:
    $$ F(x,y)e^{\overline{z}z}=\sum_{n,m\geq0}A_{nm}\left[\frac{z+\overline{z}}{2}\right]^n\left[\frac{z-\overline{z}}{2i}\right]^m$$
    By expanding the binomials, rearranging terms, and adding a convenient combinatorial factor $(n!m!)^{-1/2}$, we can rewrite the power series in terms of $z, \overline{z}$ with new coefficients $F_{nm}$:
    $$ F(x,y)e^{\overline{z}z}=\sum_{n,m\geq0}\frac{F_{nm}}{\sqrt{n!m!}}\overline{z}^nz^m$$
    If we treat $z, \overline{z}$ as independent variables, this expression involves the composition of entire functions, i.e.\@ $F[x(z, \overline{z}),y(z, \overline{z})]e^{\overline{z}z}$, so the resulting power series is entire in $z, \overline{z}$ independently. Solving for $F(x,y)$ and comparing with the power-series expansion of $Q_F(z)=\langle z|\hat{F}|z\rangle$ in eq.\@ \ref{later-Q-def} of Lemma \ref{conjugate-off-diag}, we see:
    $$ F(x,y) = e^{-\overline{z}z}\sum_{n,m\geq0}\frac{F_{nm}}{\sqrt{n!m!}}\overline{z}^nz^m = \langle z|\hat{F}|z\rangle$$
    if and only if the Fock-basis matrix elements of $\hat{F}$ are equal to the power-series coefficients, $\langle n|\hat{F}|m\rangle = F_{nm}$; since $F(x,y)$ is real, these coefficients must be Hermitian, $F_{nm}=\overline{F_{mn}}$. Thus, we \textit{define} the operator $\hat{F}$ to have these matrix elements. At this stage, it is not at all clear that they define a \textit{bounded} operator. Nonetheless, $\hat{F}$ with these specified Fock-basis matrix elements will satisfy:
    $$F(x,y) = \langle z|\hat{F}|z\rangle = \langle x+iy|\hat{F}|x+iy\rangle\;\;\text{for}\;x,y\in\mathbb{R}.$$

    We now extend $F(x,y)$ to complex arguments. From Lemma \ref{paley-wiener}, this extension is an entire function of exponential type, with growth bounded by the imaginary parts of its arguments:
    $$|F(x,y)| \leq M \exp(\tau_x|\Im x|+\tau_y|\Im y|)\;\;\text{for}\;x,y\in\mathbb{C}$$
    From eq.\@ \ref{off-diag-extension} of Lemma \ref{conjugate-off-diag}, this complex extension is given by the normalized off-diagonal coherent-state matrix elements of the same operator $\hat{F}$, defining $w=\overline{x}+i\overline{y}$ and $z=x+iy$:
    \begin{equation}\label{F-off-diag}
        F(x,y)=\frac{\langle\overline{x}+i\overline{y}|\hat{F}|x+iy\rangle}{\langle\overline{x}+i\overline{y}|x+iy\rangle} = \frac{\langle w|\hat{F}|z\rangle}{\langle w|z\rangle}\;\;\text{for}\;x,y\in\mathbb{C}
    \end{equation}
    It is essential here that the proof of Lemma \ref{conjugate-off-diag} does not depend on $\hat{F}$ being a bounded operator, only that its Fock-basis matrix elements match the power-series coefficients $F_{nm}$ as required.

    To prove boundedness, we use the Schur test \cite{Grafakos2014}, in the following form. Given a Hilbert space $L^2(X,\mu)$, we define a linear operator $T$ that acts on $f\in L^2$ via the integral kernel $K(x,y)$ as:
    $$ T[f](x)  = \int_X K(x,y)f(y)\,d\mu(y)$$
    The Schur test states that this operator will be bounded (from $L^2(X,\mu)$ to itself) if there exist two constants $A,B$ such that:
    \begin{align*}
        \int_X|K(x,y)|\,d\mu(y) &\leq A\;\;\text{for all}\;x\in X \\
        \int_X|K(x,y)|\,d\mu(x) &\leq B\;\;\text{for all}\;y\in X
    \end{align*}
    Furthermore, the operator norm will satisfy $\| T\|\leq \sqrt{AB}$.

    In what follows, we will often convert between two sets of complex variables, so it may be helpful to have a table:

    \begin{table}[h!]
    \centering
    \begin{tabular}{|c|c|c|}
        \hline
        \textbf{Purpose} & \textbf{Variable} & \textbf{Conversion formulas} \\
        \hline\hline
        Originally defined on $\mathbb{R}$, & $x$ & $x=(z+\overline{w})/2$ \\
        \cline{2-3}
        but extended to $\mathbb{C}$ & $y$ & $y=(z-\overline{w})/(2i)$ \\
        \hline\hline
        Originally defined on $\mathbb{C}$ & $z$ & $z=x+iy$ \\
        \cline{2-3}
         & $w$ & $w=\overline{x}+i\overline{y}$ \\
        \hline
        \end{tabular}
    \end{table}   

    Note that these definitions are consistent with functions that are holomorphic in $(x,y)$ and sesqui-holomorphic in $(w,z)$, i.e., holomorphic in $\overline{w}$ and $z$. In particular, we have $\overline{w}=x-iy$.

    To use the Schur test, we apply the unitary Segal-Bargmann-Fock correspondence (eq.\@ \ref{bargmann-inner-prod} of Lemma \ref{bargmann}). Consider the action of an operator $\hat{A}$ on a vector $|\psi\rangle$ in the Fock space $\mathcal{H}_F$. Each such vector corresponds to a complex function $\psi(z) = \langle z|\psi\rangle$ belonging to $L^2(\mathbb{C})$ with the following squared norm:
    $$ \langle\psi|\psi\rangle = \frac{1}{\pi}\int \overline{\psi(z)}\psi(z)\, d^2 z = \frac{1}{\pi}\int \langle\psi|z\rangle\langle z|\psi\rangle\, d^2 z ,$$
    where we recognize the resolution of the identity operator (in the weak operator topology, see eq.\@ \ref{resolution-ident} of Lemma \ref{bargmann}):
    $$ \hat{I} \simeq \frac{1}{\pi}\int |z\rangle\langle z|\, d^2 z $$
    
    The action of $\hat{A}$ on $|\psi\rangle$ is given by $\hat{A}|\psi\rangle$. Taking the inner product with the coherent-state vector $\langle w|$ and inserting the resolution of the identity, we get:
    $$ \langle w|\hat{A}|\psi\rangle = \frac{1}{\pi}\int\langle w|\hat{A}|z\rangle\langle z|\psi\rangle d^2z .$$
    On the left, we have a complex function of $w$ that corresponds to the Segal-Bargmann-Fock representation of the vector $\hat{A}|\psi\rangle$; let's label this function as $T_A[\psi](w)$. Inside the integral, we recognize that $\langle z|\psi\rangle$ is a complex function of $z$ that corresponds to the Segal-Bargmann-Fock representation of the vector $|\psi\rangle$; we already identified this function with $\psi(z)$. Furthermore, we identify the integral kernel $K_A(w,z)=\langle w|\hat{A}|z\rangle$. These substitutions give:
    $$ T_A[\psi](w) = \frac{1}{\pi}\int K_A(w,z)\psi(z) d^2z .$$
    This derivation shows that the coherent-state matrix elements $\langle w|\hat{A}|z\rangle$ define the values of the integral kernel $K_A(w,z)$ for the Segal-Bargmann-Fock representation of the operator $T_A$ corresponding to the linear operator $\hat{A}$ on the Fock space. 
    
    Thus, as expected, we have an isometric isomorphism between vectors $|\psi\rangle$ in the Fock space and functions $\psi(z)$ in $L^2(\mathbb{C})$. Likewise, we have an isometric isomorphism between linear operators $\hat{A}$ on the Fock space and operators $T_A$ defined by integral kernels $K_A(w,z)=\langle w|\hat{A}|z\rangle$ on $L^2(\mathbb{C})$. Moreover, from eq.\@ \ref{wz-to-xy} of Lemma \ref{conjugate-off-diag}, we know that: 
    $$K_A(w,z) = \langle w|\hat{A}|z\rangle = A(x,y)\langle w|z\rangle.$$ 
    
    With this isomorphism, we can apply the Schur test to show that the operator with integral kernel $K_F(w,z)=\langle w|\hat{F}|z\rangle$ is bounded on $L^2(\mathbb{C})$; this implies that the isometric operator $\hat{F}$ is bounded on the Fock space. The Schur test requires the magnitude of the kernel, $|K_F(w,z)|$, which we calculate as:
    \begin{align*}
        |K_F(w,z)| &= |\langle w|\hat{F}|z\rangle|  \\
         &= |F(x,y)\langle w|z\rangle|\qquad\text{(eq. \ref{F-off-diag})}  \\
         &= |F(x,y)||\langle w|z\rangle|  \\
         &= |F(x,y)|\sqrt{\langle w|z\rangle \langle z|w\rangle}  \\
         &= |F(x,y)|\sqrt{\exp(-|z-w|^2)}  \\
         &= |F(x,y)|\exp(-\frac{1}{2}|z-w|^2)
    \end{align*}
    Now note that $z-w = (x+iy)-(\overline{x}+i\overline{y}) = 2i(\Im x) - 2(\Im y)$, so $|z-w|^2 = 4(\Im x)^2 + 4(\Im y)^2$.  Putting it all together, we have:
    $$  |K_F(w,z)| = |F(x,y)|e^{-2(\Im x)^2-2(\Im y)^2} .$$

    To apply the Schur test, we first seek a constant $A$ that bounds the following integral $I(w)$ for all $w \in \mathbb{C}$:
    $$ I(w) \equiv \frac{1}{\pi}\int |K_F(w,z)|\,d^2z \leq A$$
    Since the complex Lebesgue measure $d^2z = d\Re z\,d\Im z$ is translation-invariant, and we are integrating over all of $\mathbb{C}$, for fixed $w$ we can shift the measure and write this integral as:
    $$ I(w) =\frac{1}{\pi}\int |K_F(w,z)|\,d^2(z-w) = \frac{4}{\pi}\int |K_F(w,z)|\,d\Im x\,d\Im y $$
    where we use $ z-w = 2i(\Im x)-2(\Im y)$, and the Jacobian gives $|\det J|=4$. Substitute our earlier expression for $|K_F(w,z)|$:
    $$ I(w) = \frac{4}{\pi}\int |F(x,y)|e^{-2(\Im x)^2-2(\Im y)^2}\,d\Im x\,d\Im y$$
    and apply Lemma \ref{paley-wiener}, $|F(x,y)| \leq M \exp(\tau_x|\Im x|+\tau_y|\Im y|)$, to obtain:
    $$ I(w) \leq \frac{4}{\pi}\int Me^{\tau_x|\Im x|+\tau_y|\Im y|}e^{-2(\Im x)^2-2(\Im y)^2}\,d\Im x\,d\Im y.$$
    for constants $\tau_x$ and $\tau_y$ that represent half of the Fourier bandwidth in each coordinate. We now separate variables and rewrite these integrals in standard form over two half-lines:
    $$ I(w) \leq \frac{4M}{\pi}\left[2\int_0^\infty e^{\tau_x r - 2r^2}\,dr\right]\left[2\int_0^\infty e^{\tau_y r - 2r^2}\,dr\right].$$
    These Gaussian integrals can be evaluated in closed form by completing the square to obtain:
    $$ 2\int_0^\infty e^{\tau_j r - 2r^2}\,dr = \sqrt{\frac{\pi}{2}}e^{\tau_j^2/8}\left[1+\text{erf}\left(\frac{\tau_j}{2\sqrt{2}}\right)\right]$$
    These calculations show convergence for any $\tau_j$, so we have achieved our goal of bounding the integral $I(w)$ by a constant $A$. Since $\mathrm{erf}(x)<1, \forall x$, we can simplify the final expression at the expense of a slightly looser bound by simply setting $\text{erf}(x)=1$:
    $$ I(w) \leq 8M e^{(\tau_x^2+\tau_y^2)/8}$$

    For the other part of the Schur test, we seek a constant B that bounds the following integral for all $z \in \mathbb{C}$:
    $$ \frac{1}{\pi}\int |K_F(w,z)|\,d^2w \leq B$$
    Since the kernel $K_F(w,z)=\langle w|\hat{F}|z\rangle$ is Hermitian, we have $|K_F(w,z)|=|K_F(z,w)|$ and thus the bound $B$ is the same as the earlier bound $A$; note the Schur test gives $\|T_F\| \leq \sqrt{AB}$.

    Putting this all together, we bound the operator norm of the integral operator $T_F$ with kernel $K_F(w,z)$ as:
    $$\|T_F\| \leq 8M e^{(\tau_x^2+\tau_y^2)/8}$$
    Since the Segal-Bargmann-Fock transform is isometric, this is also the operator norm of the operator $\hat{F}$ defined on the Fock space through the matrix elements in its power-series expansion. Thus $\hat{F}$ is a bounded, self-adjoint operator as desired, and the function $F(x,y)=\langle x+iy|\hat{F}|x+iy\rangle$ defines a dual function in the QPM.

    Extension to higher dimensions is straightforward. Given a real, bounded function $F(\mathbf{x},\mathbf{y})\leq M$ on $\mathbb{R}^{2n}$ that is band-limited, the entire argument follows by replacing the complex variables $x,y,w,z \in \mathbb{C}$ with vectors $\mathbf{x},\mathbf{y},\mathbf{w},\mathbf{z}\in \mathbb{C}^n$, and indices $n,m$ with multi-indices $\mathbf{n},\mathbf{m}$. From eq.\@ \ref{resolution-ident} of Lemma \ref{bargmann}, the resolution of the identity in $\mathbb{C}^n$ has a numerical factor of $1/\pi^n$. In calculating the bound for the Schur test on $L^2(\mathbb{C}^n)$, the Jacobian contributes a factor $4^n$, and the product of $2n$ Gaussian integrals gives a factor of $(2\pi)^n$ (using the looser bound that sets $\mathrm{erf}(x)=1$).  The final bound for the norm of $\hat{F}$, for  $F(\mathbf{x},\mathbf{y}) = \langle\mathbf{x}+i\mathbf{y}|\hat{F}|\mathbf{x}+i\mathbf{y}\rangle$ on $\mathbb{R}^{2n}$, is:
    $$ \|\hat{F}\| \leq 8^nM \exp\left[ \frac{1}{8} \sum_{j=1}^{2n}\tau_j^2\right]$$
    where we recall from Lemma \ref{paley-wiener} that the $\tau_j$'s define the  support of the Fourier transform as the box $[-\tau_1,\tau_1] \times [-\tau_2,\tau_2]\times\ldots\times[-\tau_{2n},\tau_{2n}]$.

    This is a loose bound: the constant function $F(\mathbf{x},\mathbf{y})= M$, with Fourier transform supported at the origin, is represented by a multiple of the identity operator, $\hat{F} = M\hat{I}$, which has operator norm $|M|$, not $8^n|M|$. Nonetheless, any bound is sufficient to prove that $\hat{F}$ is a bounded operator. And this bound may yield some intuition about the IPM witness functions for the Fock embedding; these functions must have operator representations with norm $\|\hat{F}\|\leq 1$. Using the Schur bound, we see that these witness functions can have arbitrary support on $\mathbb{R}^{2n}$, but their Fourier bandwidth may be limited as the operator norm could grow exponentially with the square of the bandwidth.

    We complete the proof by using Theorem 3.1 from \cite{Dryanov2003}: any bounded, uniformly continuous function on $\mathbb{R}$ can be uniformly approximated by a bounded real entire function of exponential type. The cited proof is a straightforward exercise in analytic interpolation, and extends readily to functions on $\mathbb{R}^n$. And the class of bounded real entire functions of exponential type are precisely the bounded, band-limited functions---which are all contained as dual functions in the Fock embedding QPM. Thus the dual functions are dense in $BUC(\mathbb{R}^n)$: they can uniformly approximate any bounded, uniformly continuous function on $\mathbb{R}^n$.
\end{proof}
    It is instructive to compare the Schur bound for the QPM, which uses the operator norm, with the bound for MMD using the Gaussian RKHS. Corollary \ref{gaussian-RKHS} shows that the RKHS norm is simply the Hilbert-Schmidt norm of the operator embedding in the Fock space $\mathcal{H}_F$. The space of Hilbert-Schmidt operators is much smaller than the space of bounded operators. Every Hilbert-Schmidt operator must be compact, so even the identity operator is not Hilbert-Schmidt. For the Fock space on $\mathbb{C}$, if we have a self-adjoint operator $\hat{f}$ that corresponds to a function $f(z)= \langle z|\hat{f}|z\rangle$, the squared Hilbert-Schmidt norm is given by:
    \begin{align*}
        \|\hat{f}\|_2^2&=\text{tr}(\hat{f}^\dagger\hat{f}) \\
        &= \frac{1}{\pi}\int\langle z|\hat{f}^\dagger\hat{f}|z\rangle\,d^2z \\
        &= \frac{1}{\pi^2}\iint\langle z|\hat{f}^\dagger|w\rangle\langle w|\hat{f}|z\rangle\,d^2w\,d^2z \\
        &= \frac{1}{\pi^2}\iint|K_f(w,z)|^2\,d^2w\,d^2z
    \end{align*}
    where the first line expresses the trace in terms of coherent states, the second inserts the resolution of the identity, and the third identifies the kernel $K_f(w,z)=\langle w|\hat{f}|z\rangle$. The operator $\hat{f}$ is Hilbert-Schmidt (and defines a function in the Gaussian RKHS) if and only if this kernel is square-integrable over all $w$ and $z$. In contrast, the Schur test requires only that the kernel be absolutely integrable over $z$ for each fixed $w$, and vice versa---a much weaker requirement. As a result, even on a compact space, the IPM witness functions for the QPM are a \emph{much} broader class of functions compared with the RKHS functions used in MMD. For instance, while all functions in the Gaussian RKHS must vanish at infinity, we have shown that all real, bounded, band-limited functions on $\mathbb{R}^n$ are dual functions in the QPM defined by the Fock embedding.

\subsection{Notational conventions}\label{notation}

The ``Dirac notation'' for linear algebra on a Hilbert space is quite useful for the mathematical manipulations that arise in quantum mechanics. Vectors in $\mathcal{H}$ are represented as \emph{kets}, notated $|\cdot\rangle$. The label inside the ket could represent an arbitrary vector $|\psi\rangle$, an index like $|3\rangle$ or $|n\rangle$, a coherent-state label $|z\rangle$ or $|\alpha\rangle$, or even an experimental outcome like $|\text{up}\rangle$ or $|\text{alive}\rangle$. For instance, instead of denoting an orthonormal basis as $e_n$ for indices $n\in\mathbb{N}$, one might simply write $|n\rangle$.

Dual vectors, the linear functionals in the dual space $\mathcal{H}^*$, are represented by \emph{bras}, notated $\langle\cdot|$. Every ket has a dual bra, and vice versa, so we can write $\langle\psi|$, $\langle 3|$, $\langle n|$, $\langle z|$, $\langle\alpha|$, $\langle\text{up}|$ or $\langle\text{alive}|$. The dual relationship between vectors and linear functionals is emphasized in this notation. When a vector (ket) is multiplied by a scalar, the dual of $\lambda|\psi\rangle$ is its Hermitian adjoint $\overline{\lambda}\langle\psi|$. Scalar multiplication with bras or kets is commutative: $\overline{\lambda}\langle\psi| = \langle\psi|\overline{\lambda}$.

The inner product between two vectors $(\phi,\psi)$ is written as the linear functional (bra) $\langle\phi|$ acting on the vector (ket) $|\psi\rangle$ giving the bracket (bra--ket) $\langle\phi|\psi\rangle$. The inner-product notation should be consistent with scalar multiplication of the kets, i.e.\@ $\langle\phi|\lambda|\psi\rangle = \lambda\langle\phi|\psi\rangle$ (since $\lambda$ multiplies the object to its right). This requires the inner product to be linear in the \emph{second} argument and anti-linear in the \emph{first}. That is, we would like the inner product notation to mirror the bra--ket notation:
$$(\phi,\alpha\psi+\beta\chi)=\langle\phi|\big(\alpha|\psi\rangle+\beta|\chi\rangle\big) =\alpha\langle\phi|\psi\rangle+\beta\langle\phi|\chi\rangle.$$ This is different from the usual math convention that places the dual object in the second slot, giving an inner product that is linear in the \emph{first} argument. If you have ever wondered why physicists adopt the opposite convention, it is precisely to retain consistency with Dirac notation in this fashion. In practice, writing inner products as bra--kets avoids any ambiguity about which slot is linear: scalars written next to any bra or ket simply follow the usual (commutative and associative) rules for scalar multiplication.

Linear operators are often written with ``hats'' like $\hat{A}$ or $\hat{n}$. They act on vectors to the right: $\hat{A}|\psi\rangle$ represents the action of $\hat{A}$ on the vector $|\psi\rangle$. Of course, $\hat{A}|\psi\rangle$ is also a vector, so one can act on it using a linear functional $\langle\phi|$ to give $\langle\phi|\hat{A}|\psi\rangle$. A virtue of this notation is that one can equally view this expression as the linear functional $\langle\phi|\hat{A}$ acting on the vector $|\psi\rangle$. The linear functional $\langle\phi|\hat{A}$ is the bra corresponding to the ket $\hat{A}^\dagger|\phi\rangle$, where $\hat{A}^\dagger$ represents the operator adjoint (Hermitian conjugate) of $\hat{A}$. Scalar multiplication is consistent with the action of scalar multiples of the identity operator, e.g.\@ $\langle\phi|\lambda|\psi\rangle = \langle\phi|(\lambda\hat{I})|\psi\rangle = \lambda\langle\phi|\psi\rangle$.

Dirac notation is especially useful for representing rank-1 projection operators, which represent pure states in quantum mechanics. Given a normalized vector $|\psi\rangle$, the projection operator $\hat{P}_\psi$ onto its one-dimensional subspace is simply $|\psi\rangle\langle\psi|$. This dyadic expression can be inserted anywhere you would insert the projector $\hat{P}_\psi$: for instance, the projection of a vector $|\phi\rangle$ onto this subspace is simply $\hat{P}_\psi|\phi\rangle=|\psi\rangle\langle\psi|\phi\rangle$. This notation eliminates any confusion about whether the projection of a vector $v$ along $w$ is given by $(v,w)w$ or $(w,v)w$.

This notation also shines in representing operators as sums over these dyads. For instance, given an orthonormal basis $|n\rangle$ for a separable Hilbert space, the identity operator can be written as:
$$\hat{I}=\sum_{n\in\mathbb{N}}|n\rangle\langle n|.$$
Likewise, a compact self-adjoint operator $\hat{T}$ with eigenvectors $|\phi_n\rangle$ and real eigenvalues $\lambda_n$ can be written in spectral form as:
$$\hat{T}=\sum_{n\in\mathbb{N}}\lambda_n|\phi_n\rangle\langle \phi_n|.$$
These dyadic sums yield a well-defined operator $\hat{T}$ as long as they converge between any pair of states $\langle\phi|\hat{T}|\psi\rangle$ (that is, in the weak operator topology) and are bounded as a sesquilinear form, i.e.\@ $\langle\phi|\hat{T}|\psi\rangle \leq M \|\phi\| \|\psi\|$ for all $|\phi\rangle$, $|\psi\rangle$, and some finite $M$.

Given a trace-class operator $\hat{S}$, the trace can be written using Dirac notation as:
$$\mathrm{tr}(\hat{S})=\sum_{n\in\mathbb{N}}\langle n|\hat{S}|n\rangle$$
for any orthonormal basis $|n\rangle$. The trace of a dyad $|\phi\rangle\langle\psi|$ is simply $\mathrm{tr}(|\phi\rangle\langle\psi|)=\langle\psi|\phi\rangle$, as can be derived using Dirac notation:
\begin{align*}
    \mathrm{tr}(|\phi\rangle\langle\psi|) &= \sum_{n\in\mathbb{N}}\langle n|\phi\rangle\langle\psi|n\rangle \\
    &= \sum_{n\in\mathbb{N}}\langle\psi|n\rangle\langle n|\phi\rangle  = \langle\psi|\hat{I}|\phi\rangle = \langle\psi|\phi\rangle
\end{align*}
where we recognize the expansion of the identity operator as a sum over orthonormal projectors; this is a lovely example of the utility of Dirac notation. We use this trace identity extensively. For instance, in defining the dual functions for a QPM, given an embedding that maps $x$ to the projector $\hat{\rho}_x=|x\rangle\langle x|$, and a bounded operator $\hat{F}$, the dual function is $F(x)=\mathrm{tr}(\hat{\rho}_x\hat{F})=\mathrm{tr}(|x\rangle\langle x|\hat{F})=\langle x|\hat{F}|x\rangle$. In quantum mechanics, we would describe this as the expected value of the operator (observable) $\hat{F}$ in the pure state $|x\rangle$.

Indeed, the QPM embedding is at its heart an embedding of probability measures as expected values of rank-1 projection operators. For instance, the standard Fock embedding over $\mathbb{C}$ using pure coherent states $\hat{\rho}_z=|z\rangle\langle z|$ is given by the barycenter map:
$$ T(\mu) = \int_\mathbb{C} \hat{\rho}_z\, d\mu(z)=\int_\mathbb{C} |z\rangle\langle z|\, d\mu(z) \equiv \hat{\mu} $$
which is well-defined as a strong (Bochner) integral on the Banach space $\mathcal{B}_1(\mathcal{H})$ of trace-class operators since the map $z\to \hat{\rho}_z$ is continuous from $\mathbb{C}$ to the separable space $\mathcal{B}_1(\mathcal{H})$, and the trace norm of the integrand is $\|\hat{\rho}_z\|_1=1$.

If we have a dual function $F(z)$ represented by a bounded operator $\hat{F}$, the expected value of $F(z)=\langle z|\hat{F}|z\rangle$ is given by:
$$\mathrm{tr}(\hat{F}\hat{\mu}) = \mathrm{tr}\left[\hat{F} \left(\int_\mathbb{C} |z\rangle\langle z|\, d\mu(z)\right)\right]$$
Let's work this out carefully as an exercise in Dirac notation. The Banach dual of $\mathcal{B}_1(\mathcal{H})$ is the space of bounded operators, $\mathcal{B}(\mathcal{H})$. The dual pairing is given by the trace map, $(\hat{F},\hat{\mu}) =\mathrm{tr}(\hat{F}\hat{\mu})$. Since the definition of $\hat{\mu}$ is given by a Bochner integral, it is also valid as a Pettis (weak) integral, which means that the dual pairing of $\hat{\mu} \in \mathcal{B}_1(\mathcal{H})$ with $\hat{F} \in \mathcal{B}(\mathcal{H})$ via the trace map can be brought inside the integral, i.e.:
\begin{align*}
    \mathrm{tr}\left[\hat{F} \left(\int_\mathbb{C} |z\rangle\langle z|\, d\mu(z)\right)\right]
    &= \int_\mathbb{C} \mathrm{tr}(\hat{F}|z\rangle\langle z|)\, d\mu(z)  \\
    &= \int_\mathbb{C} \langle z|\hat{F}|z\rangle\, d\mu(z) \\
    &= \int_\mathbb{C} F(z)\, d\mu(z) = \mathbb{E}_{z\sim\mu}\left[F(z)\right]
\end{align*}

Finally, just as the expansion of the identity operator in terms of an orthonormal basis shows the strength of Dirac notation, we use a similar expansion for the identity operator in terms of coherent states on $\mathbb{C}^n$ (eq.\@ \ref{resolution-ident} in Lemma \ref{bargmann}):
$$ \hat{I} \simeq \frac{1}{\pi^n}\int |\mathbf{z}\rangle\langle\mathbf{z}|\, d^{2n}\mathbf{z} $$
Here we are careful not to write this as an operator equality; it holds only in the sense of the weak operator topology when ``sandwiched'' between any pair of states $\langle\phi|\cdot|\psi\rangle$, i.e.:
$$ \langle\phi|\hat{I}|\psi\rangle = \langle\phi|\psi\rangle = \frac{1}{\pi^n}\int \langle\phi|\mathbf{z}\rangle\langle\mathbf{z}|\psi\rangle\, d^{2n}\mathbf{z} $$
by virtue of the Segal-Bargmann-Fock correspondence. This is in contrast to the infinite sums found in operator expansions in terms of rank-1 dyads, which represent a convergent series of operators in various topologies (including the weak operator topology). Here the integral representation of the identity in terms of coherent states does \emph{not} converge as an operator-valued Dunford or Pettis integral. Thus it is only well-defined when inserted between a pair of states---and this is precisely how we use it in our proofs.

\subsection{Replacing MMD with a QPM}\label{mmd2qpm}
Typical applications of MMD involve choosing $n$ i.i.d.\@ samples from a probability measure $\mu$, and $m$ samples from another measure $\nu$, and calculating the MMD between the resulting empirical measures. More generally, let
$$ \mathbb P=\sum_{k=1}^{n} p_k\delta_{x_k}, \qquad
  \mathbb Q=\sum_{\ell=1}^{m} q_\ell\delta_{y_\ell},$$
be probability measures with finite support $\{x_k\}_{k=1}^{n}$ and $\{y_\ell\}_{\ell=1}^{m}$. We wish to calculate the MMD between these discrete measures.

Let $\phi:X\to\mathcal H$ be the feature map of a characteristic kernel $k'$, written $\phi(x)=|x\rangle$ with inner product $\langle x|y\rangle = k'(x,y)$. We assume without loss of generality that the kernel is normalized, $\langle x|x\rangle = k'(x,x)=1$; if not, multiply $k'(x,y)$ by $(k'(x,x)k'(y,y))^{-1/2}$. Embed the measures by their barycenters,
$$ T(\mathbb{P}) = |\mathbb{P}\rangle=\sum_{k=1}^n p_k|x_k\rangle,\qquad T(\mathbb{Q}) = |\mathbb{Q}\rangle=\sum_{\ell=1}^m q_\ell|y_\ell\rangle$$
so MMD is the Hilbert space distance between these vectors:
$$ \text{MMD}[\mathcal{H}](\mathbb{P},\mathbb{Q})=\big\| |\mathbb{P}\rangle - |\mathbb{Q}\rangle \big\|_\mathcal{H}$$
(If $(\mathbb{P}, \mathbb{Q})$ are sampled i.i.d.\@ from $(\mu, \nu)$, this is a biased estimate of $\text{MMD}(\mu,\nu)$, i.e. V-statistic.) Writing the difference vector $|D\rangle=|\mathbb{P}\rangle - |\mathbb{Q}\rangle$ as:
\begin{equation}\label{diff-vector}
       |D\rangle = \sum_{i=1}^{N} c_i |x_i\rangle
\end{equation}
where we define $N = n+m$ and:
\begin{align}\label{diff-coeff-indices}
    c_i &= \begin{cases}
    p_i & 1\leq i\leq n,\\
    -q_{(i-n)} & n+1\leq i\leq N \end{cases} \\
    x_i &= \begin{cases}\label{diff-vector-indices}
    x_i & 1\leq i\leq n,\\
    y_{(i-n)} & n+1\leq i\leq N, \end{cases} 
\end{align}
Now MMD is the norm of this difference vector, which we can expand using Dirac notation:
\begin{align*}
    \text{MMD}[\mathcal{H}](\mathbb{P},\mathbb{Q}) &= \big\| |D\rangle \big\|_\mathcal{H} 
    = \sqrt{\langle D|D\rangle}  \\
    &= \left[\left(\sum_{i=1}^{N}c_i\langle x_i|\right)\left(\sum_{j=1}^{N}c_j|x_j\rangle\right)\right]^{1/2}  \\
    &= \left[\sum_{i,j=1}^N c_i c_j\langle x_i|x_j\rangle\right]^{1/2}  \\
    &= \left[\sum_{i,j=1}^N c_i c_j k'(x_i,x_j)\right]^{1/2}
\end{align*}
Calculating this MMD requires evaluating the kernel between every pair of points in the joint support of the measures $\mathbb{P}$, $\mathbb{Q}$, for $O(N^2)$ kernel evaluations. The matrix $G_{ij}=\langle x_i|x_j\rangle$ is simply the Gram matrix for those embedded vectors.

Now suppose that we want the norm-squared kernel instead, i.e., we seek the MMD for the kernel $k(x,y) = |k'(x,y)|^2$. With a different kernel, we obtain a different RKHS. As discussed in the proof of Theorem \ref{isomorphism}, SI Appendix, \ref{isomorphism-proof}, the RKHS for the norm-squared kernel is the Hilbert-Schmidt space $\mathcal{B}_2(\mathcal{H})$, where we define the inner product between two Hilbert-Schmidt operators as $(\hat{A},\hat{B})=\mathrm{tr}(\hat{A}^\dagger\hat{B})$. In this space, the kernel evaluates to: 
$$k(x,y) = |k'(x,y)|^2  = (\hat{\rho}_x,\hat{\rho}_y)=\mathrm{tr}(\hat{\rho}_x\hat{\rho}_y)=\langle y|x \rangle\langle x|y \rangle,$$
with the self-adjoint pure states (rank-1 projection operators) defined as $\hat{\rho}_x=|x\rangle\langle x|$ and $\hat{\rho}_y=|y\rangle\langle y|$.

In this space, instead of embedding probability measures as barycenters of embedded vectors in $\mathcal{H}$, we embed as barycenters of embedded operators (pure states) in $\mathcal{B}_2(\mathcal{H})$:
$$ T(\mathbb{P}) = \hat{\mathbb{P}}=\sum_{k=1}^n p_k|x_k\rangle\langle x_k|\,,\qquad T(\mathbb{Q}) = \hat{\mathbb{Q}}=\sum_{\ell=1}^m q_\ell|y_\ell\rangle\langle y_\ell|\,,$$
and define the MMD using the Hilbert-Schmidt norm:
$$ \text{MMD}[\mathcal{B}_2(\mathcal{H})](\mathbb{P},\mathbb{Q})=\| \hat{\mathbb{P}} - \hat{\mathbb{Q}} \|_2\, .$$
Following eq.\@ \ref{diff-vector}, we define the difference operator $\hat{D}$ as:
\begin{equation}\label{diff-op}
    \hat{D}=\hat{\mathbb{P}} - \hat{\mathbb{Q}}=\sum_{i=1}^{N} c_i |x_i\rangle\langle x_i|
\end{equation} 
using the same notation as eq.\@ \ref{diff-coeff-indices}-\ref{diff-vector-indices}; note $\mathrm{tr}(\hat{D}) = \sum_i c_i = 0$. An MMD calculation using the Hilbert-Schmidt inner product yields:
\begin{align*}
    \text{MMD}[\mathcal{B}_2(\mathcal{H})](\mathbb{P},\mathbb{Q}) &= \| \hat{D}\|_2  = \sqrt{\mathrm{tr}(\hat{D}^\dagger\hat{D})}  \\
    &= \left[\mathrm{tr}\left(\sum_{i,j=1}^{N}\Big[c_i|x_i\rangle\langle x_i|\Big]\Big[c_j|x_j\rangle\langle x_j|\Big]\right)\right]^{1/2}  \\
    &= \left[\sum_{i,j=1}^N c_i c_j\mathrm{tr}\big(|x_i\rangle\langle x_i|x_j\rangle\langle x_j|\big)\right]^{1/2}  \\
    &= \left[\sum_{i,j=1}^N c_i c_j\langle x_j|x_i\rangle\langle x_i|x_j\rangle\right]^{1/2}  \\
    &= \left[\sum_{i,j=1}^N c_i c_j |k'(x_i,x_j)|^2\right]^{1/2}   \\
    &= \left[\sum_{i,j=1}^N c_i c_j k(x_i,x_j)\right]^{1/2} 
\end{align*}
which is identical to the MMD calculated using the norm-squared kernel $k(x,y) = |k'(x,y)|^2$.
Equivalently, if the eigenvalues of the self-adjoint operator $\hat{D}$ are given by $\lambda_i$, the MMD is given by:
$$ \text{MMD}[\mathcal{B}_2(\mathcal{H})](\mathbb{P},\mathbb{Q}) = \| \hat{D}\|_2 = \left[\sum_{i=1}^N \lambda_i^2 \right]^{1/2} $$
We calculate the QPM using the same embedding and the same difference operator $\hat{D}$ (eq.\@ \ref{diff-op}) but with trace distance instead of Hilbert-Schmidt distance. With the conventional normalization factor of $1/2$ (giving distances in [0,1]), we have:
\begin{align*}
    \text{QPM}(\mathbb{P},\mathbb{Q}) &= \tfrac{1}{2}\| \hat{\mathbb{P}} - \hat{\mathbb{Q}} \|_1 = \tfrac{1}{2}\|\hat{D}\|_1 \\
    &= \tfrac{1}{2}\mathrm{tr}(|\hat{D}|)  \\
    &= \tfrac{1}{2}\mathrm{tr}\left[(\hat{D}^2)^{1/2}\right] 
\end{align*}
where $|\hat{D}|\equiv (\hat{D}^2)^{1/2}$ represents the (operator) square root of the nonnegative operator $\hat{D}^2$. Here we see the central computational challenge in using QPM. Unlike MMD, which can be calculated using $O(N^2)$ kernel evaluations, the operator square root required for QPM needs eigenvalues of the operator $\hat{D}$, with typical complexity $O(N^3)$ for an $N\times N$ matrix. Recall that the eigenvalues of $|\hat{D}|$ are the absolute values of the eigenvalues of $\hat{D}$, and the trace of an operator is the sum of its eigenvalues. Thus, if the eigenvalues of $\hat{D}$ are given by $\lambda_i$, the QPM is:
$$ \text{QPM}(\mathbb{P},\mathbb{Q}) = \tfrac{1}{2}\| \hat{D}\|_1 = \tfrac{1}{2}\sum_{i=1}^N |\lambda_i| . $$

Now we need the eigenvalues of the difference operator $\hat{D}=\sum_{i=1}^{N} c_i |x_i\rangle\langle x_i|$. This is not directly amenable to matrix calculation, as the embedded vectors $|x_i\rangle$ do not define an orthonormal basis, and the only quantities we can easily calculate are kernel evaluations $k'(x,y)=\langle x|y\rangle$. However, we can use a bit of linear algebra to help. Choose a set of $N$ orthonormal basis vectors $|e_i\rangle \in \mathcal{H}$, whose linear span includes the span of the $|x_i\rangle$, and define the operators $\hat{V}= \sum_{i=1}^N |x_i\rangle\langle e_i|$ and $\hat{C}=\sum_{j=1}^N c_j|e_j\rangle\langle e_j|$ so that the difference operator can be expressed as:
\begin{align*}
    \hat{D}=\hat{V}\hat{C}\hat{V}^\dagger&=\left[\sum_{i=1}^N |x_i\rangle\langle e_i|\right]\left[\sum_{j=1}^N c_j|e_j\rangle\langle e_j|\right]\left[\sum_{k=1}^N |e_k\rangle\langle x_k|\right]  \\
    &= \sum_{i=1}^N c_i|x_i\rangle\langle x_i| = \hat{D}
\end{align*}
using $\langle e_j|e_i\rangle=\delta_{ij}$. As operators on $\mathrm{span}\{|e_i\rangle\}$, their eigenvalues are preserved under cyclic permutations, so the eigenvalues of $\hat{D}=\hat{V}\hat{C}\hat{V}^\dagger$ are the same as those of $\hat{C}\hat{V}^\dagger\hat{V}$. Now observe that the operator $\hat{G}=\hat{V}^\dagger\hat{V}$ defines the Gram matrix in the $|e_i\rangle$ basis:
$$ \hat{G}=\hat{V}^\dagger\hat{V}=\left[\sum_{i=1}^N |e_i\rangle\langle x_i|\right]\left[\sum_{j=1}^N |x_j\rangle\langle e_j|\right]=\sum_{i,j=1}^N|e_i\rangle\langle x_i|x_j\rangle\langle e_j|$$
i.e.\@ $G_{ij}=\langle e_i|\hat{G}|e_j\rangle=\langle x_i|x_j\rangle$. So we can equivalently find the eigenvalues of the operator $\hat{C}\hat{G}$, whose matrix elements are given by the matrix product $M=CG$, i.e.\@ $M_{ij}=c_iG_{ij}=c_i\langle x_i|x_j\rangle$. Although this is a non-Hermitian matrix, its eigenvalues will be real since they are identical to those of the Hermitian operator $\hat{D}$. They can be calculated using a general eigenvalue solver; any imaginary parts that arise due to rounding errors should be discarded to preserve the essential relationships $\mathrm{tr}(\hat{D})=\sum_i \lambda_i =\sum_i c_i = 0$.

Nonetheless, calculating eigenvalues of non-Hermitian matrices can be computationally costly and numerically unstable. For better numerics, we factor the Gram matrix as $G=HH^\dagger$, e.g.\@ using Cholesky factorization or the matrix square root. Using invariance under cyclic permutation, the eigenvalues of $CG= CHH^\dagger$ are the same as those of $H^\dagger CH$, which is a Hermitian matrix with numerically stable (real) eigenvalues. These are the eigenvalues we want---they are the eigenvalues of the difference operator $\hat{D}$.

We now have a simple drop-in ``recipe'' for replacing MMD with QPM in any calculation that uses finite discrete measures (e.g.\@ empirical measures). MMD already requires the Gram matrix $G_{ij}=\langle x_i|x_j\rangle=k'(x_i,x_j)$. Instead of calculating MMD using the quadratic form $c_i c_j G_{ij}$, factor the Gram matrix into $G=HH^\dagger$, define the coefficient matrix $C=\mathrm{diag}(c_i)$, and find the eigenvalues $\lambda_i$ of the Hermitian matrix $H^\dagger CH$. The QPM will be $\tfrac{1}{2}\sum_{i=1}^N|\lambda_i|$.

There is one subtle point, however. The original MMD used the kernel $k'(x,y)=\langle x|y\rangle$, while the QPM calculation effectively uses the norm-squared kernel $k(x,y)=|k'(x,y)|^2=\langle y|x\rangle\langle x|y\rangle$. The difference is clear when we consider the distance between two Dirac measures (point masses), $\mathbb{P}=\delta_x$ and $\mathbb{Q}=\delta_y$, embedded as vectors in a Hilbert space $\mathcal{H}$: $\delta_x \to |x\rangle$ and $\delta_y \to |y\rangle$. Assuming (as we have throughout) that all embedded vectors are normalized, $\langle x|x\rangle = \langle y|y\rangle = 1$, the MMD in $\mathcal{H}$ is given by:
\begin{align*}
    \text{MMD}[\mathcal{H}](\delta_x,\delta_y) &= \big\| |x\rangle-|y\rangle \big\|_\mathcal{H}  \\
    &= \bigg[\big(\langle x| - \langle y|\big)\big(|x\rangle - |y\rangle\big)\bigg]^{1/2}  \\
    &= \big[ 2 - \langle x|y\rangle - \langle y|x\rangle \big]^{1/2}  \\
    &= \big[ 2 - k'(x,y) - k'(y,x) \big]^{1/2}
\end{align*}
In contrast, we obtain the norm-squared kernel if we embed the Dirac measures as operators in the Hilbert-Schmidt space $\mathcal{B}_2(\mathcal{H})$: $\delta_x \to \hat{\delta}_x = |x\rangle\langle x|$ and $\delta_y \to \hat{\delta}_y = |y\rangle\langle y|$. The MMD is given by the norm of the difference operator $\hat{D}=\hat{\delta}_x - \hat{\delta}_y=|x\rangle\langle x| - |y\rangle\langle y|$:
\begin{align*}
    \text{MMD}[\mathcal{B}_2(\mathcal{H})](\delta_x,\delta_y) &= \big\|  |x\rangle\langle x| - |y\rangle\langle y|  \big\|_2  \\
    &= \big[ 2 - 2|\langle x|y\rangle|^2 \big]^{1/2}  \\
    &= \sqrt{2}\big[ 1 - |k'(x,y)|^2 \big]^{1/2}  \\
    &= \sqrt{2}\big[ 1 - k(x,y) \big]^{1/2} 
\end{align*}
The QPM for these measures can be calculated by diagonalizing the difference operator $\hat{D}=|x\rangle\langle x| - |y\rangle\langle y|$. The eigenvalues of this operator are: $\lambda_\pm = \pm\sqrt{1-|\langle x|y\rangle|^2}$, so we calculate:
\begin{align*}
    \text{QPM}(\delta_x,\delta_y) &= \tfrac{1}{2}\big\|  |x\rangle\langle x| - |y\rangle\langle y|  \big\|_1 \\
    &= \tfrac{1}{2}\sum_{i=1}^N|\lambda_i| = \big[ 1 - |\langle x|y\rangle|^2 \big]^{1/2}  \\
    &= \big[ 1 - |k'(x,y)|^2 \big]^{1/2}  \\
    &= \big[ 1 - k(x,y) \big]^{1/2} \\
    & = \frac{1}{\sqrt{2}} \mathrm{MMD}[\mathcal{B}_2(\mathcal{H})](\delta_x,\delta_y)
\end{align*}
Up to a factor of $\sqrt{2}$, therefore, the QPM between point masses is identical to the MMD calculated using the Hilbert-Schmidt metric, i.e.\@ with the norm-squared kernel. Conversely, MMD using a kernel $K$ will give the same point-mass metric as QPM using the square-root kernel $\sqrt{K}$. This is an important detail when replacing MMD with QPM.

Thus, given a calculation that uses MMD, one should first ask if the MMD kernel can be written as the norm-square of another kernel. For many common kernels, including the Gaussian, Laplacian, and inverse-multiquadric kernels, the square root of each of these kernels is also a strictly positive definite kernel. In that case, the best recipe for replacing it with QPM is as follows:
\begin{enumerate}
    \item We are given an MMD calculation between two probability measures with finite support, i.e., 
    $$ \mathbb P=\sum_{k=1}^{n} p_k\delta_{x_k}, \qquad
    \mathbb Q=\sum_{\ell=1}^{m} q_\ell\delta_{y_\ell},$$

    \item For convenience, use a single set of indices from $1$ to $N=n+m$ with coefficients and points of support defined as:
    $$c_i = \begin{cases}
    p_i & 1\leq i\leq n,\\
    -q_{(i-n)} & n+1\leq i\leq N \end{cases} $$
    $$ x_i = \begin{cases}
    x_i & 1\leq i\leq n,\\
    y_{(i-n)} & n+1\leq i\leq N, \end{cases} $$

    \item Where MMD was calculated from the Gram matrix $\langle x_i|x_j\rangle = K(x_i,x_j)$, calculate a new Gram matrix using the square root of the kernel, i.e.\@ $\langle x_i|x_j\rangle=\sqrt{K(x_i,x_j)}$. One must ensure here that $\sqrt{K}$ is also a valid (positive definite) kernel.
    
    \item Factor this new Gram matrix $G=HH^\dagger$ using Cholesky factorization. Define $M=H^\dagger CH$ where $C=\mathrm{diag}(c_i)$, and find the eigenvalues $\lambda_i$ of this Hermitian $M$. If some points are coincident, $x_i=x_j$ for $i\neq j$, Cholesky will fail; one could eliminate these coincidences to restore a positive-definite Gram matrix or use a general eigensolver and extract the (real) eigenvalues $\lambda_i$ of the non-Hermitian matrix $CG$.
    
    \item Then $\text{QPM}(\mathbb{P},\mathbb{Q})=\tfrac{1}{2}\sum_{i=1}^N |\lambda_i|$ with conventional normalization. If isometry between the MMD and QPM metrics for point masses is desired instead, multiply the QPM by $\sqrt{2}$.
\end{enumerate}

In case the kernel $K$ used in the existing MMD calculation cannot be written as the norm-square of another kernel, there are two reasonable options. One could proceed using the kernel $K$ directly, in which case the QPM base metric effectively uses $|K|^2$ as its kernel. This could have undesirable effects: for instance, the effective bandwidth of $|K|^2$ will be smaller than that of $K$. 

Alternately, one could find a kernel whose norm-square closely approximates the original kernel over the relevant range. We use the second approach in the GMMN example below: the original kernel was defined as a sum of Gaussians with a range of widths, giving a ``fat-tailed'' kernel that was effective in that particular application. This sum of Gaussians cannot be written as the square of any kernel, but the fat-tailed kernel can be matched with an inverse-multiquadric kernel, which does have a valid positive-definite kernel square root. In practice, this second approach will tend to follow the original geometry more closely.

\subsection{QPMs in Generative Moment Matching Networks (GMMNs)}\label{GMMN-application}
As an example of the ``drop-in'' replacement of MMD with QPM, we turn to a classic generative modeling task that relies on MMD to measure the distance between probability distributions: the Generative Moment Matching Network (GMMN) \cite{Li2015GMMN}. We reproduce this result and then replace the MMD metric with QPM, keeping all else constant (architecture, optimizer, training schedule, etc.).

The original GMMN implementation used a 5-layer fully-connected feed-forward network to map a random latent vector $z$ into a data space, such as the space of MNIST images \cite{lecun1998gradient}. This transforms a fixed probability distribution $P_z$ on the latent space into a distribution on the data space, and training seeks to match that distribution to that of the target dataset, using the MMD metric as its objective. This is a conceptually simple and appealing idea: the distribution of generated images should match as closely as possible to the actual distribution. In the specific example from \cite{Li2015GMMN}, they begin with a uniform distribution on $[-1,1]^{10}$ in a 10-dimensional latent space and pass it through 5 layers using ReLU activation with 64, 256, 256, and 1024 hidden dimensions respectively, ending with sigmoid activation on 784 dimensions to match the $28 \times 28$ grayscale MNIST images with pixel values in $[0,1]$. The net was trained by choosing random batches of 1000 images from the dataset, generating 1000 images from the network, and backpropagating to minimize MMD. The code is available from the original author at \url{https://github.com/yujiali/gmmn}.

We first re-implemented this network architecture using PyTorch and reproduced the results from the paper. Then, from this baseline, we swapped QPM in place of MMD as the objective function to run a side-by-side comparison of these two probability metrics. This presented two key challenges. First, the kernel used for MMD in the original paper was an equal mixture of 6 Gaussians with a range of values for $\sigma^2=\{2,5,10,20,40,80\}$. (Note that the original author's code uses the variable name {\tt sigma} for these values but careful examination shows that they actually represent $\sigma^2$ in the usual parameterization of the Gaussian kernel.) As discussed at the end of SI Appendix, \ref{mmd2qpm}, although the sum of Gaussians is a valid kernel, its square root is not. Thus, to preserve the geometry in the embedding space, we created a kernel that closely matched the original mixture of Gaussians over the data space. We used a generalized inverse-multiquadric kernel with two parameters: a length scale $\lambda$ and exponent $\alpha$, yielding:
$$ k(\mathbf{x},\mathbf{y}) = \left[1+\frac{\|\mathbf{x}-\mathbf{y}\|^2}{2\alpha \lambda^2}\right]^{-\alpha}$$
With the best-fit parameters, this kernel yielded an excellent match to the ``fat-tailed'' mixture of Gaussians and reproduced the original results nicely. Now we could use the (valid) square root of this kernel as the basis for the QPM, ensuring that the fundamental embedding geometry is preserved.

The second challenge was related to the MNIST dataset, which has many pairs of images that are very similar to one another, with Euclidean distance $\|\mathbf{x}-\mathbf{y}\|$ near zero, or, equivalently, a kernel inner product $k(\mathbf{x},\mathbf{y})$ very close to 1 (our kernels are normalized to $k(\mathbf{x},\mathbf{x})=1$). As a result, even if the Gram matrix for 1000 samples from the dataset is not truly singular, rounding error can cause Cholesky factorization to fail. We tried two approaches that yielded identical results: (1) numerical regularization by adding a small multiple of the identity ($\epsilon=10^{-4}$) to the Gram matrix, and (2) using a generalized eigensolver directly with the non-Hermitian matrix $CG$ ($C = \mathrm{diag}(c_i)$; see the end of SI Appendix, \ref{mmd2qpm}, for notation). The first approach is faster and is probably superior for most applications: the stochastic nature of the sampling and generation process is likely to outweigh any small errors in calculating the ``true'' eigenvalues (and their gradients). Nonetheless, for a robust implementation of QPM, the generalized eigensolver offers a backup in the event that Cholesky fails.

To generate the images in Figure \ref{fig:generated}, we trained for 200 epochs on the MNIST training set (50,000 images) using batches of 1000 and the PyTorch \texttt{AdamW} optimizer with a learning rate of 0.005 and default parameters. (See \url{https://doi.org/10.5281/zenodo.17024967} for results and code.) A single training loop trained two models together with the only difference being the objective function (QPM versus MMD). The MMD results are qualitatively similar to those from \cite{Li2015GMMN}, with many visible artifacts that do not improve beyond 200 epochs. In contrast, the images generated using QPM appear much more similar to the MNIST samples.

After training, we used a kernel two-sample test \cite{Muandet2017} to assess whether the generated images were similar to images from the dataset. We compared a batch of 1000 real images (from a held-out test set) with a batch of 1000 generated images. The baseline distance between these batches was measured using the same metrics as for training (MMD or QPM). Then, to determine if this distance was statistically significant, we randomly permuted all 2000 images into two new batches of 1000 each. We repeated this shuffling process 1000 times, calculating the distance for each random split. If the real vs. generated images were from the same distribution (null hypothesis), then the baseline distance should be typical of these random distances. This would imply that, based on that metric, the trained network met its goal of matching generated images to the true distribution.

Given $N=1000$ permutations, if $r$ of them exceed the baseline distance, the continuity-corrected $p$-value is $(r+1)/(N+1)$. Repeating each two-sample test for 30 distinct batches and taking the average, we find $p \approx 0.006$ for the MMD images and $p \leq 0.001$ for the QPM images. In both cases, these metrics can detect that the generated images are not drawn from the same distribution as the MNIST dataset---confirming our visual inspection.

For a more challenging test, we trained a GMMN on the CelebA dataset \cite{liu2015faceattributes} (\url{https://mmlab.ie.cuhk.edu.hk/projects/CelebA.html}), pre-processed to yield $64 \times 64$ RGB images---a data space of 12,288 dimensions. This dataset contains over 200,000 images with faces in various orientations, lighting conditions, and backgrounds. In place of the MLP used for MNIST, we adopted the convolutional generator {\tt DCGAN\_G} from \url{https://github.com/martinarjovsky/WassersteinGAN}. (Note there is nothing ``adversarial'' about our approach; there is no discriminator network.) The {\tt DCGAN\_G} network starts with a multivariate standard normal distribution on a 100-dimensional latent space and uses 5 convolutional layers with BatchNorm and ReLU after each hidden layer and tanh activation for the final layer (with pixel values scaled to $[-1,1]$). Two models were trained together (training set of 180,000 images, 50 epochs, batches of 1000), comparing QPM with MMD as objective functions. We used a simple Gaussian kernel to avoid the ``kernel square root'' problem (the square root of a Gaussian kernel is also a valid kernel), and found that the diversity of CelebA images ensured that Cholesky factorization typically succeeded without regularization.

Examining Figure \ref{fig:generated}, the images generated using these two different metrics appear quite different. Many of the MMD images have bizarre artifacts, while all the QPM images show clear facial structure. More strikingly, the two-sample test for the MMD-generated images, using the same metric, yielded $p \approx 0.23$ (average of 30 runs); this metric could not reject the null hypothesis that the generated images were drawn from the same distribution as CelebA. This result shows the limitations of MMD in high-dimensional spaces: the training has optimized the MMD-distance between the data and the generated images, and statistically the two sets of images are similar based on this metric, but the images are not even remotely similar. In contrast, QPM readily distinguishes its own generated images or the MMD-generated images from the CelebA data, with $p \leq 0.001$ in both cases. We note that the MMD metric became similarly saturated (unlike QPM) for Gaussian kernels with $\sigma$ ranging from 15 up to 80 (the median distance between pairs of image vectors drawn from the dataset); the results in Figure \ref{fig:generated} took $\sigma$ as half this median distance.

While training the CelebA generator using QPM alone doesn't generate photorealistic images, at least this metric can detect that the generated images are different from the dataset (using the two-sample test). In contrast, optimizing for MMD yielded very poor images that were nonetheless viewed by that metric as being statistically indistinguishable from the training data. This simple experiment shows that MMD can be unreliable in high dimensions where QPM is robust. The improved performance of QPM is most likely due to its rich class of dual functions (Theorem \ref{uniform-approx}, SI Appendix, \ref{uniform-approx-proof}). In contrast, the dual functions for MMD must vanish at infinity, so its witness functions (unit ball in the RKHS) effectively have limited support, even on compact domains. It is likely that these witness functions can only detect differences along a limited number of dimensions, leaving large parts of the high-dimensional space effectively unmonitored.

\end{document}